\documentclass[pdflatex,sn-nature]{sn-jnl}

\DeclareFontFamily{U}{mathx}{\hyphenchar\font45}
\DeclareFontShape{U}{mathx}{m}{n}{
      <5> <6> <7> <8> <9> <10>
      <10.95> <12> <14.4> <17.28> <20.74> <24.88>
      mathx10
      }{}
\DeclareSymbolFont{mathx}{U}{mathx}{m}{n}
\DeclareMathAccent{\widebar}{0}{mathx}{"73}
\DeclareFontSubstitution{U}{mathx}{m}{n}

\usepackage{bm}
\usepackage{mathtools}
\usepackage{amsthm}
\usepackage{amssymb}
\usepackage{graphicx}
\usepackage{algorithm}
\usepackage{algpseudocode}
\usepackage[most]{tcolorbox}
\usepackage{enumitem}

\newtheorem*{assumption*}{Assumption}
\newtheorem{theorem}{Theorem}
\newtheorem{lemma}{Lemma}
\newtheorem{corollary}{Corollary}

\makeatletter
\newtheorem*{rep@theorem}{\rep@title}
\newcommand{\newreptheorem}[2]{%
\newenvironment{rep#1}[1]{%
 \def\rep@title{#2 \ref{##1}}%
 \begin{rep@theorem}}%
 {\end{rep@theorem}}}
\makeatother

\newreptheorem{theorem}{Theorem}
\newreptheorem{lemma}{Lemma}
\newreptheorem{corollary}{Corollary}
\newreptheorem{proposition}{Proposition}

\newcommand\ci{\perp\!\!\!\perp}

\begin{document}

\title[Global Interpretability via Automated Preprocessing]{Global Interpretability via Automated Preprocessing: A Framework Inspired by Psychiatric Questionnaires}

\affil[1]{\orgdiv{Department of Biomedical Informatics}, \orgname{University of Pittsburgh}}

\author*[1]{\fnm{Eric} \sur{V. Strobl}}

\abstract{
Psychiatric questionnaires are highly context sensitive and often only weakly predict subsequent symptom severity, which makes the prognostic relationship difficult to learn. Although flexible nonlinear models can improve predictive accuracy, their limited interpretability can erode clinical trust. In fields such as imaging and omics, investigators commonly address visit- and instrument-specific artifacts by extracting stable signal through preprocessing and then fitting an interpretable linear model. We adopt the same strategy for questionnaire data by decoupling preprocessing from prediction: REFINE (Redundancy-Exploiting Follow-up-Informed Nonlinear Enhancement) uses follow-up information to learn a nonlinear, item-aligned preprocessing of baseline measurements, and then maps these stabilized baseline items to future severity through a linear coefficient matrix. This construction constrains where nonlinearity enters rather than restricting the end-to-end predictive function class, allowing the prognostic relationship to remain globally interpretable through a single coefficient matrix rather than through post hoc local attributions. In the population, the resulting predictor recovers the Bayes-optimal conditional mean, so global linear interpretability does not require sacrificing nonlinear predictive flexibility. In experiments, REFINE outperforms other interpretable approaches while preserving clear global attribution of prognostic factors across psychiatric and non-psychiatric longitudinal prediction tasks.}

\maketitle

\section{Introduction}
Psychiatric questionnaires, such as the Hamilton Depression Rating Scale (HAM-D) for depression \citep{Hamilton67} and the Positive and Negative Syndrome Scale (PANSS) for psychosis \citep{Kay87}, measure multiple symptom dimensions rather than a single construct. Predicting future questionnaire responses is therefore important for characterizing longitudinal changes in mental illness and informing clinical decision-making. Crucially, the clinical objective is often broader than forecasting a single aggregate endpoint such as a total score. Patients with similar total scores can exhibit markedly different symptom profiles, and these differences can imply distinct monitoring priorities and clinical concerns \citep{Strobl25_JAFD,Strobl24_SV}. A clinically useful prognostic model should therefore forecast the full symptom vector at future follow-up times in a unified and interpretable manner, so that clinicians can assess how all symptom dimensions are expected to change.

Satisfying this objective is challenging because item-level questionnaire responses are often noisy, and different symptoms can follow distinct, complex, and nonlinear trajectories over time. As a result, accurate forecasting often requires flexible nonlinear models \citep{Quinn24}. However, these models can be difficult to interpret, which limits clinicians' ability to assess validity and can erode trust even when predictive performance is strong. Investigators often address this limitation with local attribution tools (e.g., SHAP values \citep{Lundberg17}) or instance-specific effect models (e.g., varying-coefficient models \citep{Hastie93}), but these approaches can produce patient-specific explanations that vary substantially across individuals. When the apparent importance of predictors shifts markedly from one patient to another, it becomes difficult to extract a coherent global picture of the prognostic relationship. Moreover, this problem is amplified in psychiatric questionnaires because the outcome is a vector of correlated items, and explanations can fluctuate not only across patients but also across symptom dimensions. 

In many other scientific domains, investigators address noisy nonlinear relationships through \emph{preprocessing} and then fit a \textit{globally} interpretable linear model whose coefficients have the same meaning across patients. In neuroimaging, for example, preprocessing typically removes non-brain tissue (e.g., dura, skull, and cerebrospinal fluid), suppresses high-frequency artifacts, and often applies spatial smoothing to promote local continuity across neighboring voxels \citep{Smith04}. In transcriptomics, researchers remove outliers, discard low-quality samples, and normalize read counts to account for the multinomial-like behavior of sequencing data \citep{Conesa16}. From a machine learning perspective, such steps are highly nonlinear but not essential in the infinite-sample limit, because sufficiently expressive models can, in principle, learn comparable transformations directly from raw inputs. Scientifically, however, preprocessing serves an additional purpose beyond sample efficiency: it can make downstream linear models substantially more interpretable. For instance, after spatial smoothing, linear models trained on neuroimaging data often produce spatially coherent coefficients, making parameters easier to align with known functional organization \citep{Jabakhanji22}.

These observations raise a parallel question in psychiatry: can we preprocess questionnaire items so that simple, globally interpretable linear models achieve predictive performance comparable to nonlinear models? We seek to confine nonlinearity to the preprocessing step, leaving a single linear map from preprocessed items to future outcomes whose coefficients have the same interpretation across patients. This preprocessing challenge is distinctive, however, because questionnaires offer few natural preprocessing cues compared with many other data modalities---for example, they lack spatial locality to exploit and do not admit obvious count-based normalizations.
\begin{tcolorbox}[enhanced,frame hidden,breakable]
In this paper, we carefully address the above question through the following contributions:
\begin{enumerate}[leftmargin=*]
\item We formalize two minimal requirements for a clinically interpretable preprocessing operator: it should preserve \textbf{longitudinal redundancy} and \textbf{item-level meaning}. These requirements provide principled guidance for preprocessing questionnaire items.
\item We introduce REFINE (Redundancy-Exploiting Follow-up-Informed Nonlinear Enhancement), which learns such a preprocessor by first using future measurements as privileged training-time supervision to enforce \emph{longitudinal redundancy}.
\item REFINE then confines nonlinearity to a preprocessing step in a way that maintains item alignment, so each coordinate retains its original \emph{item-level meaning}.
\item We show that the mapping from preprocessed items to future items is \textbf{exactly linear} and globally interpretable, while still recovering the \textbf{Bayes-optimal} conditional mean and therefore imposing no restriction on the population predictive target.

\item Empirically, REFINE improves upon other interpretable methods by placing flexibility in an explicit preprocessing step rather than imposing functional restrictions (e.g., linearity or tree structure) on the predictive pipeline.
\item Finally, using a non-psychiatric longitudinal dataset, we show that REFINE extends beyond questionnaires to other multivariate clinical
measurements.
\end{enumerate}
Overall, REFINE captures complex nonlinear relationships among multiple outcomes without sacrificing interpretability: it learns an item-aligned nonlinear preprocessing operator that exploits longitudinal redundancy to preserve task-relevant information, then delegates prediction to an exactly linear, globally interpretable, and Bayes-optimal decoder (Figure \ref{fig:main_idea}).
\end{tcolorbox}

\begin{figure}
    \centering
    \includegraphics[width=1\linewidth]{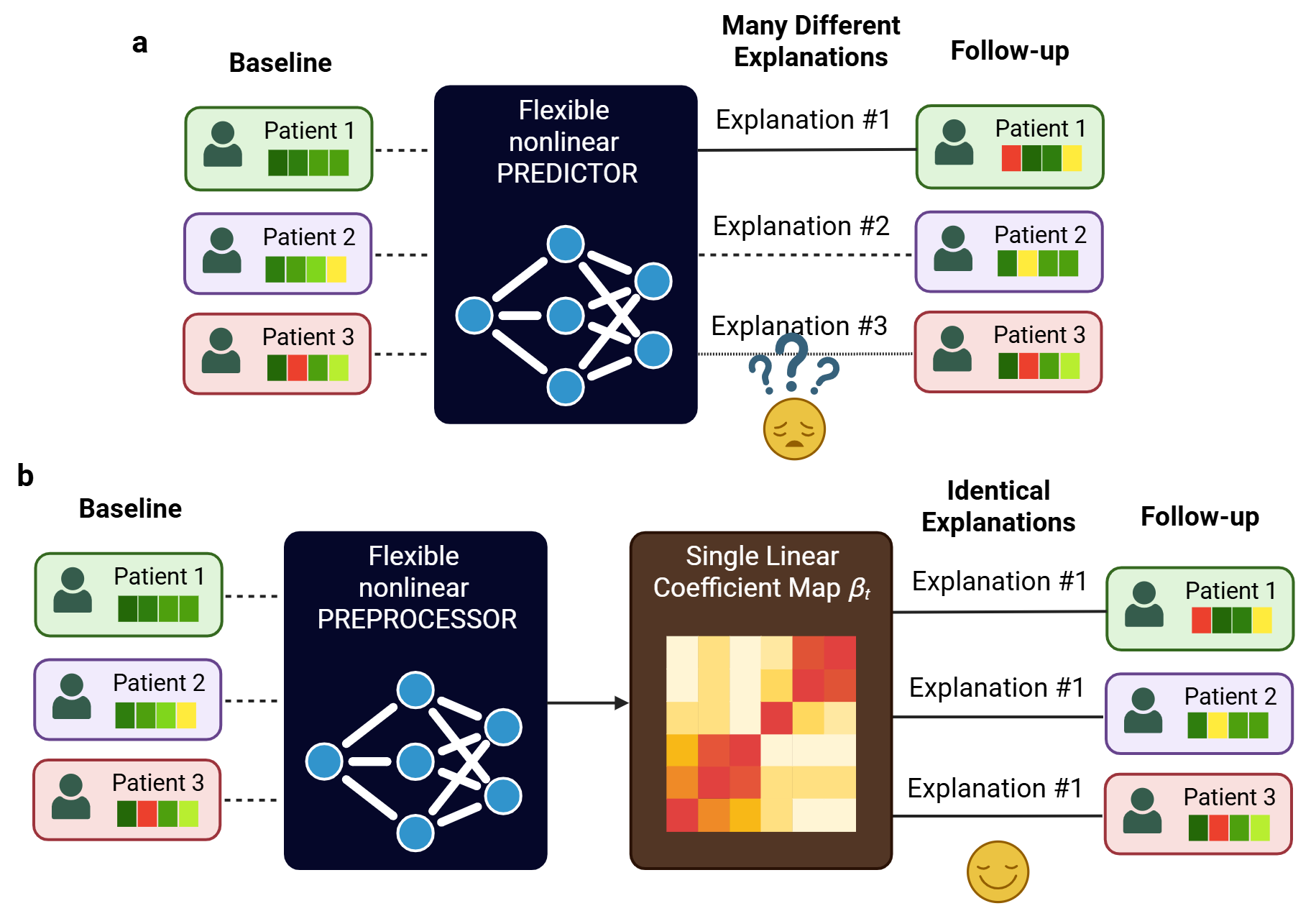}
\caption{\textbf{Main idea.} (a) A standard approach uses a flexible nonlinear model to predict follow-up outcomes from baseline measurements, but interpretation typically relies on post hoc local attribution methods, producing explanations that can differ across patients. (b) REFINE instead places nonlinear flexibility in an item-aligned preprocessing step. Crucially, this is not an arbitrary nonlinear representation, but one whose coordinates retain their interpretation as preprocessed versions of the original items. Prediction is then reserved for a single linear coefficient map \(\beta_t\) shared across patients. Patients may therefore have different preprocessed values and predictions, while the prognostic relationship itself retains the same global interpretation for everyone. In this way, REFINE preserves nonlinear predictive flexibility without sacrificing global interpretability.}
\label{fig:main_idea}

\end{figure}

\section{Related Work}

A vast literature studies modality-specific preprocessing strategies (e.g., imaging, omics, electronic health records, wearables) which we do not review here. In psychiatric research, questionnaire ``preprocessing'' is often implemented procedurally rather than formalized algorithmically, via structured interview protocols, rater training, double coding, discrepancy checks, and adjudication \citep{Rohan16}. A complementary psychometric tradition models item responses through latent-variable measurement models such as item response theory, which summarizes many items through one or more latent traits and supports denoised scoring and downstream inference \citep{Gorter15}. However, these latent traits add an additional interpretive layer and are often treated as scientific constructs in their own right, rather than as preprocessed representations of the original item-level measurements.

Interpretable machine learning for questionnaire-style tabular data has largely pursued transparency through instance-specific attribution of nonlinear models. Widely used local-attribution methods include SHAP values \citep{Lundberg17}, gradient-based approaches \citep{Sundararajan17}, and perturbation-based explanations \citep{Ribeiro16}. In high-dimensional questionnaire settings with many outcome variables, however, these methods can become difficult to use in practice. Explanations vary across patients \textit{and} across outcomes, which substantially increases interpretive burden. Moreover, attempts to ``globalize'' local attributions---for example, by averaging across patients or summarizing with aggregate statistics---are typically heuristic and may obscure clinically meaningful heterogeneity \citep{Lundberg20}.

A second line of work seeks transparency by constraining models to operate in the original item space. Representative examples include sparse linear models \citep{tibshirani1996lasso,zou2005elasticnet}, rule-based models \citep{letham2015brl,angelino2018corels}, and additive models such as GAMs and GA$^2$Ms \citep{hastie1990gam,lou2013ga2m}. Related hybrid approaches attach an interpretable linear component to a nonlinear model---for example, concept bottleneck models \citep{Koh20} or neural parameterizations of varying-coefficient models that produce instance-specific linear effects \citep{Richman23}. These approaches improve transparency, but they do not learn an explicit preprocessing operator that stabilizes measurements while preserving item-level meaning.


The closest method that we found to our approach also seeks to preserve global interpretability by producing an intermediate representation aligned with the input feature space \citep{Shi25}. Their model learns a nonlinear mapping that reconstructs an interpretable feature representation and then applies a linear predictor for classification; when trained end-to-end, the objective couples reconstruction fidelity and predictive accuracy, so improving one can degrade the other. In contrast, our preprocessing target is defined by longitudinal redundancy: in the population, the denoised representation and subsequent linear transformation recover the conditional mean exactly, so the method does not require tuning an explicit trade-off between preprocessing fidelity and predictive performance.

\section{Minimal Criteria for a Preprocessor}

Before introducing REFINE, we first clarify what we mean by preprocessing. Although preprocessing goals vary across domains, two requirements recur whenever preprocessing must support scientific interpretation as well as prediction. First, preprocessing should \textbf{retain information that is reproducible over time}. Because longitudinal redundancy is inherently symmetric, a useful baseline representation should remain predictive of future assessments, and future assessments should retain enough shared signal to reconstruct baseline-derived quantities. This criterion encourages preprocessing to suppress visit-specific fluctuations without discarding prognostically relevant structure. Importantly, forward predictive accuracy alone (or backward reconstruction alone) does not suffice: a representation can succeed in one direction by leveraging visit-specific artifacts---such as rater changes, shifts in administration context, or altered privacy (e.g., a family member present during the interview)---rather than isolating the stable component that persists across time.

Second, preprocessing should \textbf{preserve item-wise meaning}. A preprocessor may borrow strength across items---for example, by denoising each item using information from correlated items---but it should not alter what each coordinate represents. In particular, the output should remain aligned with the original item definitions, so that downstream coefficients can be interpreted directly in terms of the questionnaire items.

These principles mirror how preprocessing is used in other scientific pipelines. In genomics and other -omics settings, normalization and quality control aim to suppress
technical artifacts while preserving gene-level meaning. In imaging, artifact removal
and spatial smoothing reduce measurement noise and increase reproducibility, yet
analysts still interpret each coordinate as the same voxel (or anatomically defined
location) after preprocessing. Our goal is to achieve the same effect for questionnaire data: suppress transient measurement artifacts while maintaining an item-aligned representation that supports transparent downstream inference.

\section{Algorithm}

\subsection{Strategy}

Let $\bm Y_0$ denote the baseline vector, decomposed as $\bm Y_0=(\bm X_0,\bm Z)$, where $\bm X_0\in\mathbb R^{d}$ contains the questionnaire items at baseline and $\bm Z\in\mathbb R^{q}$ contains additional baseline covariates (e.g., age and sex). Further let $\bm X_t\in\mathbb R^{d}$ denote the follow-up questionnaire items at time $t$. We assume i.i.d. samples and centered variables throughout. REFINE targets prediction of $\bm X_{t>0}$ from baseline measurements $(\bm X_0,\bm Z)$.

In general, the conditional expectation can be nonlinear and observations include idiosyncratic variation:
\begin{equation}\label{eq:NL}
\bm X_t \;=\; \mathbb{E}\!\left(\bm X_t \mid \bm X_0,\bm Z\right) \;+\; \varepsilon_t,
\end{equation}
where \(\mathbb{E}(\varepsilon_t\mid \bm X_0,\bm Z)=0\). REFINE exploits redundancy across time to construct a representation in which this conditional mean factorizes as
\[
\mathbb{E}\!\left(\bm X_t \mid \bm X_0,\bm Z\right)=h_t(\bm X_0,\bm Z)\,\beta_t,
\]
where \(h_t(\bm X_0,\bm Z)\in\mathbb{R}^{d}\) captures nonlinear transformations that yield stabilized baseline item values and \(\beta_t\in\mathbb{R}^{d\times d}\) is a linear coefficient matrix mapping stabilized baseline items to follow-up items. This yields the partially linear model
\begin{equation}\label{eq:decomp}
\bm X_t \;=\; h_t(\bm X_0,\bm Z)\,\beta_t \;+\; \varepsilon_t.
\end{equation}
We design \(h_t(\bm X_0,\bm Z)\) to serve solely as an \emph{automatically learned preprocessing} step that preserves item-level meaning in the original item space, learned from data without manual feature engineering. Crucially, \(h_t\) is learned from longitudinal information: the follow-up assessment at time \(t\) provides supervision not only for the prediction target \(\mathbb{E}(\bm X_t \mid \bm X_0,\bm Z)\), but also for the preprocessing representation \(h_t(\bm X_0,\bm Z)\) itself. The primary scientific target remains the linear coefficients \(\beta_t\). As a result, after preprocessing, the problem reduces to a globally interpretable linear mapping, mirroring the role of preprocessing in many scientific workflows.

\subsection{Learning How to Preprocess}\label{sec:pred}
To obtain the decomposition in~\eqref{eq:decomp}, REFINE uses follow-up measurements during training. Suppose we observe follow-up questionnaire responses at times \(t\in\{t_1,\dots,t_T\}\). On the training set, we treat each \(\bm X_t\) as privileged information and regress baseline items on follow-up items:
$\bm X_0 \;=\; \bm X_t\,\bm B_t \;+\; \varepsilon_{0,t},$
where \(\bm B_t\in\mathbb{R}^{d\times d}\) is assumed invertible. 

Least squares estimation yields the follow-up--informed proxy
\[
\widebar{\bm X}_0^{(t)} \;=\; \bm X_t\,\bm B_t
\]
for each time point. We then set \(\beta_t \coloneqq \bm B_t^{-1}\). This choice implies the corresponding reconstruction
\[
\bm X_t \;=\; \widebar{\bm X}_0^{(t)}\,\beta_t.
\]
Thus, the map \(\widebar{\bm X}_0^{(t)}\mapsto \bm X_t\) is linear by construction.

We do not observe \(\bm X_t\) at test time, and therefore we cannot form \(\widebar{\bm X}_0^{(t)}=\bm X_t\bm B_t\) on the test set. Instead, for each \(t\) we learn a baseline estimator of \(\widebar{\bm X}_0^{(t)}\) by modeling the conditional expectation
\[
h_t(\bm X_0,\bm Z)\;=\;\mathbb{E}(\widebar{\bm X}_0^{(t)}\mid \bm X_0,\bm Z),
\]
using a flexible nonlinear regression method (e.g., random forests, gradient boosting, or neural networks).

The above construction intuitively makes each \(h_t\) an explicit \emph{supervised preprocessor}. The proxy \(\widebar{\bm X}_0^{(t)}\) retains, for each baseline item, the component that is predictable from assessments at time \(t\), since \(\widebar{\bm X}_0^{(t)}=\bm X_t \bm B_t\) is the linear reconstruction of baseline responses from follow-up data. The function \(h_t(\bm X_0,\bm Z)=\mathbb{E}(\widebar{\bm X}_0^{(t)}\mid \bm X_0,\bm Z)\) then learns to recover this temporally reproducible component from baseline measurements, allowing additional baseline covariates $\bm Z$ to inform preprocessing while keeping the downstream prognostic relationship exactly linear. Importantly, \(h_t\) targets a proxy defined in the original item space rather than an arbitrary latent embedding. Consequently, each coordinate of \(h_t(\bm X_0,\bm Z)\) corresponds to a stabilized version of the corresponding baseline item, preserving a direct item-wise interpretation. We make this intuition rigorous in Section \ref{sec:theory}.

\subsection{Obtaining the Final Linear Transformation}
REFINE finally combines the learned preprocessing map with a globally interpretable linear transformation to predict future outcomes via Equation \eqref{eq:decomp}. The resulting predictor targets the conditional mean of future questionnaire responses because we can write:
\[
h_t(\bm X_0,\bm Z)\beta_t
\;=\;
\mathbb{E}(\widebar{\bm X}_0^{(t)}\mid \bm X_0,\bm Z)\beta_t
\;=\;
\mathbb{E}(\widebar{\bm X}_0^{(t)}\beta_t\mid \bm X_0,\bm Z)
\;=\;
\mathbb{E}(\bm X_t\mid \bm X_0,\bm Z),
\]
where the third equality uses \(\widebar{\bm X}_0^{(t)}\beta_t=\bm X_t\). Hence the mapping \((\bm X_0,\bm Z) \mapsto h_t(\bm X_0,\bm Z)\beta_t\) recovers the Bayes-optimal predictor \(\mathbb{E}(\bm X_t\mid \bm X_0,\bm Z)\) in the population setting.

We also note that REFINE can naturally accommodate follow-up missingness under a missing-at-random (MAR) assumption \citep{Rubin76}, where missingness at time \(t\) depends on observed baseline variables \((\bm X_0,\bm Z)\). Let \(R_t\) denote the indicator that \(\bm X_t\) is observed. If \(R_t \ci \bm X_t \mid (\bm X_0,\bm Z)\), then the complete-case conditional mean equals the target conditional mean, \(\mathbb E(\bm X_t \mid \bm X_0,\bm Z,R_t=1)=\mathbb E(\bm X_t \mid \bm X_0,\bm Z)\). Models are estimated separately at each time point using subjects with \(R_t=1\), and both the proxy target \(\widebar{\bm X}_0^{(t)}=\bm X_t \bm B_t\) and decoder \(\bm\beta_t=\bm B_t^{-1}\) are learned from that same observed subset. Consequently, although the estimated reconstruction matrix \(\bm B_t\) may reflect the observed-case distribution, this dependence cancels algebraically in the final predictor. We specifically have
\[
h_t(\bm X_0,\bm Z)\beta_t = \mathbb E(\bm X_t \bm B_t \mid \bm X_0,\bm Z,R_t=1)\bm B_t^{-1}
\;=\;
\mathbb{E}(\bm X_t\mid \bm X_0,\bm Z, R_t=1).
\]
So under MAR conditional on $(\bm X_0, \bm{Z}),$
\[
\mathbb{E}(\bm X_t\mid \bm X_0,\bm Z,R_t=1) = \mathbb{E}(\bm X_t\mid \bm X_0,\bm Z).
\]
Thus, under MAR conditional on baseline predictors, positivity of follow-up observation, and invertibility of $\bm B_t$ in the observed-case population, REFINE remains consistent for the prognostic target $\mathbb E(\bm X_t \mid \bm X_0,\bm Z)$ without requiring imputation.

\subsection{Algorithm}

\begin{algorithm}[b]
\caption{REFINE: Redundancy-Exploiting Follow-up-Informed Nonlinear Enhancement}
\begin{algorithmic}[1]
\Statex \textbf{Input:} baseline matrix
\(\bm Y_0=(\bm X_0,\bm Z)\); follow-up data
\(\{(\bm X_t)\}_{t=1}^T\); nonlinear multivariate learner
\(\mathcal A\); new baseline data
\((\bm X_0^{\mathrm{test}},\bm Z^{\mathrm{test}})\)
\Statex \textbf{Output:} fitted model
\(\{(\widehat B_t,\widehat\beta_t,\widehat h_t)\}_{t=1}^T\);
predictions \(\widehat{\bm X}_t\)
\For{$t \in \{t_1,\dots,t_T\}$}
    \State 
    $\widehat B_t, \widebar{\bm{X}}^{(t)}_0 \leftarrow$ ordinary least squares of $\bm{X}_0$ on $\bm{X}_t$ \label{alg:REFINE:OLS}
    \State $\widehat\beta_t \leftarrow \widehat B_t^{-1}$ \label{alg:REFINE:inverse}
    \State $\widehat h_t \leftarrow$ fit non-linear predictor of $\widebar{\bm{X}}^{(t)}_0$ on $(\bm{X}_0, \bm{Z})$ \label{alg:REFINE:non-linear}
    \State \(\widehat{\bm H}_t \leftarrow \widehat h_t([\bm X_0^{\mathrm{test}},\bm Z^{\mathrm{test}}])\) \label{alg:REFINE:H}
    \State \(\widehat{\bm X}_t \leftarrow \widehat{\bm H}_t\widehat\beta_t\) \label{alg:REFINE:pred}
\EndFor
\label{alg:REFINE}
\end{algorithmic}
\end{algorithm}

We provide pseudocode for REFINE in Algorithm \ref{alg:REFINE}. REFINE takes as input baseline measurements consisting of all questionnaire items and nuisance covariates, follow-up item measurements across time points, and a user-specified nonlinear multivariate learner. For each follow-up time point, REFINE first fits ordinary least-squares (OLS) models (Line \ref{alg:REFINE:OLS}) to obtain the follow-up-informed proxy \(\widebar{\bm X}_0^{(t)}\) and the reconstruction matrix \(\widehat{B}_t\). It then computes the linear decoder \(\widehat{\beta}_t\) by inverting \(\widehat{B}_t\) (Line \ref{alg:REFINE:inverse}). Next, REFINE trains the nonlinear learner to predict \(\widebar{\bm X}_0^{(t)}\) from \((\bm X_0,\bm Z)\) (Line \ref{alg:REFINE:non-linear}), yielding the preprocessing map \(\widehat{h}_t\).

After this procedure, the fitted model \(\{(\widehat B_t,\widehat\beta_t,\widehat h_t)\}_{t=1}^T\) can be applied directly to new baseline data. Specifically, for test-set baseline inputs \((\bm X_0^{\mathrm{test}}, \bm Z^{\mathrm{test}})\), REFINE first computes the stabilized representation via \(\widehat h_t([\bm X_0^{\mathrm{test}}, \bm Z^{\mathrm{test}}])\) (Line \ref{alg:REFINE:H}) and then maps it to predicted follow-up items using the learned linear decoder, producing \(\widehat{\bm X}_t = \widehat h_t([\bm X_0^{\mathrm{test}}, \bm Z^{\mathrm{test}}])\,\widehat\beta_t\) (Line \ref{alg:REFINE:pred}). Thus, although REFINE leverages follow-up information during training to learn the preprocessing operator and decoder, inference on new data depends only on baseline inputs.

We provide a detailed time-complexity analysis of REFINE in Appendix~\ref{supp:time_complex}. The overall training complexity scales as
\[
O\!\left(T\big(n d^2 + d^3 + \mathcal{C}_{\mathrm{fit}}(n,p,d)\big)\right),
\]
where \(\mathcal{C}_{\mathrm{fit}}(n,p,d)\) denotes the training cost of the chosen multivariate nonlinear learner, and $T$ the total number of follow-up time points. In our implementation, we instantiate the nonlinear preprocessor with random forests. The random-forest component scales approximately as \(n\log n\) in sample size, while REFINE's linear calibration contributes the cubic dependence on \(d\). In practice, because questionnaire dimensions are typically modest, this overhead remains small. We will empirically see that REFINE is among the fastest methods we evaluated.

\section{Theoretical Results} \label{sec:theory}

This section formalizes the intuitive rationale behind REFINE by deriving guarantees aligned with clinical prediction settings. We establish four results. First, we characterize a natural class of pipelines that preprocess first and then predict linearly, and show that REFINE belongs to this class. Second, when we require the preprocessing step to preserve item-level meaning, the optimal pipeline becomes unique and coincides exactly with REFINE. Third, we show that our preferred decoder estimator based on matrix inversion is statistically well justified and exhibits favorable convergence properties. Finally, we prove that the end-to-end estimator converges at the rate determined by the slower of the nonlinear learner and the root-$n$ estimation of the linear decoder.

\subsection{Optimality of Preprocessing}
We formalize longitudinal redundancy by asking how well follow-up items linearly predict baseline items. For a fixed follow-up time $t$, define the population least-squares reconstruction matrix
\[
B_t = \arg\min_{B\in\mathbb R^{d\times d}} \ \mathbb E\!\left[\|\bm X_0-\bm X_t B\|_2^2\right].
\]
Then $\bm X_t B_t$ is the best linear proxy for $\bm X_0$ available from the follow-up questionnaire. REFINE uses this proxy to define a supervised stabilization target in the original item coordinates.

Importantly, $\widebar{\bm X}_0^{(t)}$ lives in the \emph{original item coordinates}.
Because the squared-error objective separates by columns, the $j$th column of $B_t$ solves
\[
B_{t,\cdot j} = \arg\min_{b\in\mathbb R^{d}} \; \mathbb E\!\left[(X_{0,j}-\bm X_t b)^2\right],
\]
so the $j$th coordinate of $\widebar{\bm X}_0^{(t)}=\bm X_tB_t$ is exactly the population least-squares predictor of the baseline item $X_{0,j}$ from $\bm X_t$.
Thus, each coordinate of the proxy has a direct interpretation as a stabilized version of the corresponding baseline item, while still allowing the reconstruction to borrow strength across items through multivariate regression.

REFINE uses this item-aligned proxy as a supervised stabilization target. However, even if the final predictor is Bayes-optimal, the internal split into a ``preprocessor'' and a ``linear decoder'' need not be unique. If we insist on a pipeline of the form
\[
(\bm X_0,\bm Z)\ \xrightarrow{\ g\ }\ \mathbb R^d\ \xrightarrow{\ \text{linear }\beta\ }\ \mathbb R^d,
\]
then we can change coordinates in the intermediate $d$-dimensional space without changing the final prediction. The next theorem characterizes this non-uniqueness. All proofs are provided in Appendix~\ref{sup:proofs}.

\begin{theorem}\label{thm:preproc}
Assume $B_t$ and $\mathrm{Cov}(m_t(\bm X_0,\bm Z))$ are invertible, where
$m_t(\bm X_0,\bm Z):=\mathbb E(\bm X_t\mid \bm X_0,\bm Z)$.
Let $g:\mathbb R^{d+q}\to\mathbb R^{d}$ denote the preprocessor and let
$\beta\in\mathbb R^{d\times d}$. Then
$g(\bm X_0,\bm Z)\beta = m_t(\bm X_0,\bm Z)$ almost surely (a.s.)
if and only if $B_t\beta$ is invertible and
\[
g(\bm X_0,\bm Z) = m_t(\bm X_0,\bm Z)\,(B_t\beta)^{-1}B_t \quad\text{a.s.}
\]
\end{theorem}

\noindent\textbf{Remark.} The theorem says that Bayes-optimal factorizations are generally not unique: an invertible change of coordinates in the intermediate space produces a different $(g,\beta)$ with the same final predictor. This is an interpretability issue, because the coordinates of $g(\bm X_0,\bm Z)$ are what we want to call “preprocessed items.”

To fix the coordinate system, we impose a simple anchoring condition: decoding should do nothing beyond undoing the redundancy map induced by $B_t$ so that $B_t\beta = I_d$. Anchoring ensures that, if we had access to the oracle proxy $\widebar{\bm X}_0^{(t)}=\bm X_t B_t$, then decoding would return $\bm X_t$ exactly, i.e., $\widebar{\bm X}_0^{(t)}\beta=\bm X_t$. Under this anchoring, the Bayes-optimal factorization becomes unique:

\begin{corollary}\label{cor:unique}
Under Theorem~\ref{thm:preproc} and $B_t\beta = I_d$, the unique Bayes-optimal pair is $(g,\beta)=(h_t,B_t^{-1})$, where
$h_t(\bm X_0,\bm Z)=\mathbb E(\bm X_t B_t\mid \bm X_0,\bm Z)$.
\end{corollary}
\noindent\textbf{Remark.}
When $B_t$ is invertible, $\beta_t=B_t^{-1}$ is a one-to-one linear map from the item-anchored proxy space to $\bm X_t$. Therefore the resulting preprocess--then--linearly--predict pipeline keeps the intermediate representation aligned with baseline item meaning while expressing prognosis through a single globally interpretable coefficient matrix.

\subsection{Justification of Inversion and Convergence Rate}

We next justify how REFINE estimates the linear decoder in finite samples. Fix $t$ and suppose the population factorization in \eqref{eq:decomp} holds with
$\beta_t = B_t^{-1}$.
Let $\widehat\beta_{t,\mathrm{inv}}:=\widehat B_t^{-1}$. As a natural alternative to inversion, one may instead estimate the decoder by a second-stage OLS regression of $\bm X_t$ on the learned features $\widehat H_t$ constructed from the same sample, yielding $\widehat\beta_{t,\mathrm{OLS}}$.
Write $\widehat h_t = h_t + r_{n,t}$ and set $a_{n,t}:=\|r_{n,t}\|_{2,n_t}$. Moreover let $U_i = X_{t,i} - h_i \beta_t$. The following statement holds:

\begin{lemma}\label{lem:invbeta}
Assume the following moment and regularity conditions for a fixed $t$: 
(i) $\mathbb{E}\|h_i\|_2^4<\infty$, $\mathbb{E}\|U_i\|_2^4<\infty$ and $\mathbb E(U_i\mid \bm X_{0,i},\bm Z_i)=0$; 
(ii) $\mathbb{E}(h_i^\top h_i)$ is positive definite; 
(iii) $a_{n,t} \in o_p(1)$; and
(iv) $\mathbb{E}\|\bm X_{0,i}\|_2^4<\infty$, $\mathbb{E}\|\bm X_{t,i}\|_2^4<\infty$, and $\mathbb{E}(\bm X_{t,i}^\top \bm X_{t,i})$ is positive definite.
Then:
\[
\|\widehat\beta_{t,\mathrm{OLS}}-\beta_t\|_F
\in
O_p\!\big(n_t^{-1/2} + a_{n,t}\big),
\qquad
\|\widehat\beta_{t,\mathrm{inv}}-\beta_t\|_F
\in
O_p(n_t^{-1/2}).
\]
\end{lemma}
\noindent \textbf{Remark.} The key message is practical: if $h_t$ is learned with a flexible nonlinear method, its estimation error $a_{n,t}$ shrinks at a nonparametric rate (often slower than $n_t^{-1/2}$). If we fit the decoder by same-sample OLS on $\widehat H_t$, that nonparametric error propagates into $\widehat\beta_{t,\mathrm{OLS}}$. In contrast, estimating $\beta_t$ by inverting $\widehat B_t$ avoids this additional dependence on $a_{n,t}$ and achieves a parametric, root-$n_t$ rate. This motivates our choice to calibrate the linear decoder via inversion rather than refitting it on learned features. Moreover, we will see that these rates matter in practice in Section \ref{sec:emp}. 

With the above lemma in hand, we can bound the end-to-end error of REFINE.

\begin{theorem}\label{thm:convergence_rate}
Consider the same assumptions as Lemma~\ref{lem:invbeta} for a fixed $t$, and let
$\widehat\beta_t=\widehat\beta_{t,\mathrm{inv}}=\widehat B_t^{-1}$. Then
$\|\widehat H_t\,\widehat\beta_t - H_t\beta_t\|_{2,n_t}
\in
O_p\!\big(n_t^{-1/2}+a_{n,t}\big).$
\end{theorem}
\noindent \textbf{Remark.} This theorem shows that REFINE behaves like a standard two-stage procedure in which the overall rate is controlled by (i) a parametric term from estimating the linear reconstruction and (ii) a nonparametric term from learning the stabilizer. In particular, REFINE does not pay an extra statistical price for keeping the prognostic map exactly linear: when the nonlinear learner converges more slowly than $n_t^{-1/2}$, its rate remains the dominant term.

\section{Empirical Results} \label{sec:emp}
\subsection{Algorithms}
We compared REFINE, using 1000 random forest trees in the nonlinear preprocessing step, against the following four algorithms:
\begin{enumerate}[leftmargin=*]
\item \textbf{Attribution-based Interpretable Classification Neural Network} (AICNN) \citep{Shi25}: the most similar method to REFINE that learns nonlinear embeddings, reconstructs them into a one-hot-aligned interpretable feature-value space, and then applies a linear head for interpretability. The model is trained with a joint reconstruction and prediction loss trade-off. We set the number of epochs to 50, the batch size to 128, $\lambda$ to the author recommended value of 0.1, and otherwise used default hyperparameters.

\item \textbf{Gaussian Process Boosting} (GPBoost) \citep{Sigrist22}: a joint learning algorithm for non-linear mixed-effects models that fits the fixed-effects mean function with stage-wise boosting while iteratively estimating grouped random-effect covariance parameters.  We learned variable importance with TreeSHAP \citep{Lundberg20}. We set the number of rounds to 2000, the number of early stopping rounds to 50, and otherwise used default hyperparameters. This serves as a mixed-model baseline improved over standard XGBoost.

\item \textbf{Mixed Generalized Additive Model Computation Vehicle} (MGCV) \citep{Wood25}: fits generalized additive and mixed models by representing smooth terms with penalized basis expansions and then alternating/embedding penalized iterative reweighted least squares coefficient updates with automatic smoothing-parameter estimation. We used default hyperparameters. This serves as our additive model baseline.

\item \textbf{Extreme Gradient Boosting} (XGBoost) \citep{Chen16}: a popular boosting algorithm which builds an ensemble of decision trees stage-wise by fitting each new tree to the gradients of the loss to improve prediction. We also learned variable importance with TreeSHAP. We set the number of rounds to 500, the number of early stopping rounds to 30, and the learning rate to 0.05 to maximize performance. This serves as our standard SHAP-based local-attribution baseline.

\end{enumerate}
AICNN, like REFINE, preserves item-level meaning, but it does not leverage follow-up measurements to learn a longitudinally optimized nonlinear representation. MGCV improves interpretability by constraining the predictor to an additive functional form, but this limits its ability to capture more general nonlinear dependencies. GPBoost and XGBoost rely on TreeSHAP to summarize variable importance. SHAP provides instance-specific feature attributions, but in settings with multiple outcomes these explanations can quickly become overwhelming because they vary across patients for each outcome. As a result, investigators typically approximate global interpretability using heuristic aggregates such as mean absolute attribution \citep{Lundberg20}. REFINE instead concentrates nonlinearity in a supervised preprocessing step and then fits a single global linear model, so the primary interpretive object is a unified coefficient matrix rather than a heuristic aggregate of local explanations. Taken together, these distinctions make REFINE the only method in our comparison that jointly enforces longitudinal redundancy and item-level meaning within a single globally interpretable model.

We also compared REFINE against \textbf{two ablated variants}. In the first, we estimated each \(\beta_t\) by linear regression rather than by inverting \(B_t\), to evaluate the practical implications of Lemma \ref{lem:invbeta}. In the second, we replaced the nonlinear preprocessor \(h_t\) with ordinary least squares, to test whether the nonlinear preprocessing step is necessary for REFINE to achieve its full predictive performance. This second ablation also reduces to a simple linear baseline, because composing two linear mappings (baseline \(\to\) proxy and proxy \(\to\) follow-up) yields a single linear map from baseline measurements to follow-up items. 

Finally, because the original implementation of REFINE used random forests for nonlinear preprocessing, we examined whether performance depended strongly on the \textbf{choice of the nonlinear learner}. We replaced the random forest preprocessor with either bagged $k$-nearest neighbors (100 bags, $k=10$) or a bagged Nadaraya--Watson estimator (100 bags, Gaussian kernel, $k=50$). We chose these alternatives because they estimate the preprocessing map by local averaging of observed training targets rather than by extrapolating into unsupported regions of the feature space. This non-extrapolating behavior is particularly useful in REFINE, because errors in the learned preprocessing map are subsequently multiplied by the inverse decoder $\widehat B_t^{-1}$; extreme preprocessing values produced out of support could therefore be amplified in the final prediction.

\subsection{Metrics}

We evaluated all algorithms using five complementary metrics. First, we measured \textbf{(forward) correlation} between predicted and observed item scores at each follow-up time point, which provides a standardized, item-wise measure of regression accuracy. Second, because our primary contribution is to improve interpretability via nonlinear preprocessing, we quantified \textbf{backward correlation}: for each method, we attempted to reconstruct the baseline-derived quantities---item-aligned representations (REFINE, AICNN), additive components (MGCV), or SHAP attributions (GPBoost, XGBoost)---from the non-baseline item vectors using ridge regression. High backward correlation indicates that these baseline-derived representations preserve stable, meaningful signals across timepoints. Methods that effectively exploit longitudinal redundancy should therefore achieve strong performance in both directions, yielding high forward and backward correlation. We additionally report forward and backward \textbf{root mean squared error} (RMSE), pooled across items, as a conventional measure of overall regression accuracy.

Fourth, we quantified preservation of item-level meaning using the \textbf{cosine similarity} between each method's baseline-to-follow-up contribution matrix and its diagonal component. A higher value indicates that contributions are concentrated near the diagonal---i.e., for follow-up assessments close to baseline, item \(j\) is driven primarily by its corresponding baseline item \(j\) rather than by a mixture of other items. We do not expect perfect alignment (cosine similarity \(=1\)) in practice, because true symptom trajectories can include cross-item interactions and correlated change. However, higher cosine similarity is informative at \emph{near-baseline follow-up time points}, because cross-item dynamics and accumulated change naturally reduce diagonal dominance at longer horizons. 

Fifth, because REFINE constructs the linear decoder using the unregularized inverse $\widehat{\bm\beta}_t=\widehat{\bm B}_t^{-1}$, we quantified the \textbf{realized amplification} induced by the inverse decoder at each follow-up time:
\[
A_t^{\mathrm{real}}
=
\frac{
\left\|
\widehat{\bm H}_{t,c}\widehat{\bm\beta}_t
\right\|_F
}{
\left\|
\widehat{\bm H}_{t,c}
\right\|_F
}
=
\frac{
\left\|
\widehat{\bm H}_{t,c}\widehat{\bm B}_t^{-1}
\right\|_F
}{
\left\|
\widehat{\bm H}_{t,c}
\right\|_F
},
\]
where $\widehat{\bm H}_{t,c}$ denotes the column-centered version of
$\widehat{\bm H}_t$. This metric quantifies how strongly the
inverse decoder amplifies the variation present in the stabilized item values
learned by REFINE. Unlike $1/\sigma_{\min}(\widehat{\bm B}_t)$, which gives worst-case amplification over all directions, or the condition number, which characterizes relative conditioning, realized amplification measures the amplification that the fitted inverse decoder actually induces on the variation present in REFINE's learned stabilized representation $\widehat{\bm H}_t$. Values near one indicate no net rescaling by the inverse decoder. We regarded values above 10 as concerning in the present setting because the stabilized baseline items and corresponding follow-up items retain the same original measurement scales; an order-of-magnitude increase in variation would therefore be difficult to explain as ordinary scale restoration alone and could indicate substantial sensitivity to the unregularized inverse.

We evaluated all algorithms on each dataset using \textbf{100 bootstrap resamples}. For each resample, models were trained on the bootstrap-selected observations and evaluated on the corresponding out-of-bootstrap held-out observations. Statistical comparisons were performed using \textbf{paired, within-bootstrap hypothesis tests}. For the forward--backward Pareto analysis, we compared REFINE with the empirical comparator frontier in each bootstrap sample: we computed REFINE's gain over the best forward-performing comparator and its gain over the best backward-performing comparator, defined the frontier-crossing margin as the larger of these two gains, and tested whether the mean margin exceeded zero using a one-sided paired bootstrap \(t\)-test. For item alignment, we compared REFINE's first-follow-up cosine similarity with the highest first-follow-up cosine similarity achieved by any comparator in the same bootstrap sample, again using a one-sided paired bootstrap \(t\)-test. Significance was assessed after Bonferroni correction, and all claims below held at the corrected level \(0.05/6\), reflecting comparisons with four primary algorithms and two ablated variants.

\subsection{Clinical High Risk for Psychosis}
We first evaluated all methods on the NAPLS-3 dataset (NDA ID: 2275) \citep{Addington22}. The prediction task involves estimating the longitudinal progression of prodromal psychosis symptoms, measured by the Scale of Prodromal Symptoms (SOPS), at months 2, 4, 6, 8, 12, 18, and 24 from baseline SOPS items together with age and sex assigned at birth. SOPS is a 19-item clinical rating scale that quantifies early psychosis-spectrum symptomatology and is widely used to identify individuals at clinical high risk (CHR) for subsequent psychotic disorders \citep{Mcglashan01}. The dataset includes 309 individuals with complete baseline data, including age, sex assigned at birth, and all SOPS item scores.

\begin{figure}
    \centering
    \includegraphics[width=0.9\linewidth]{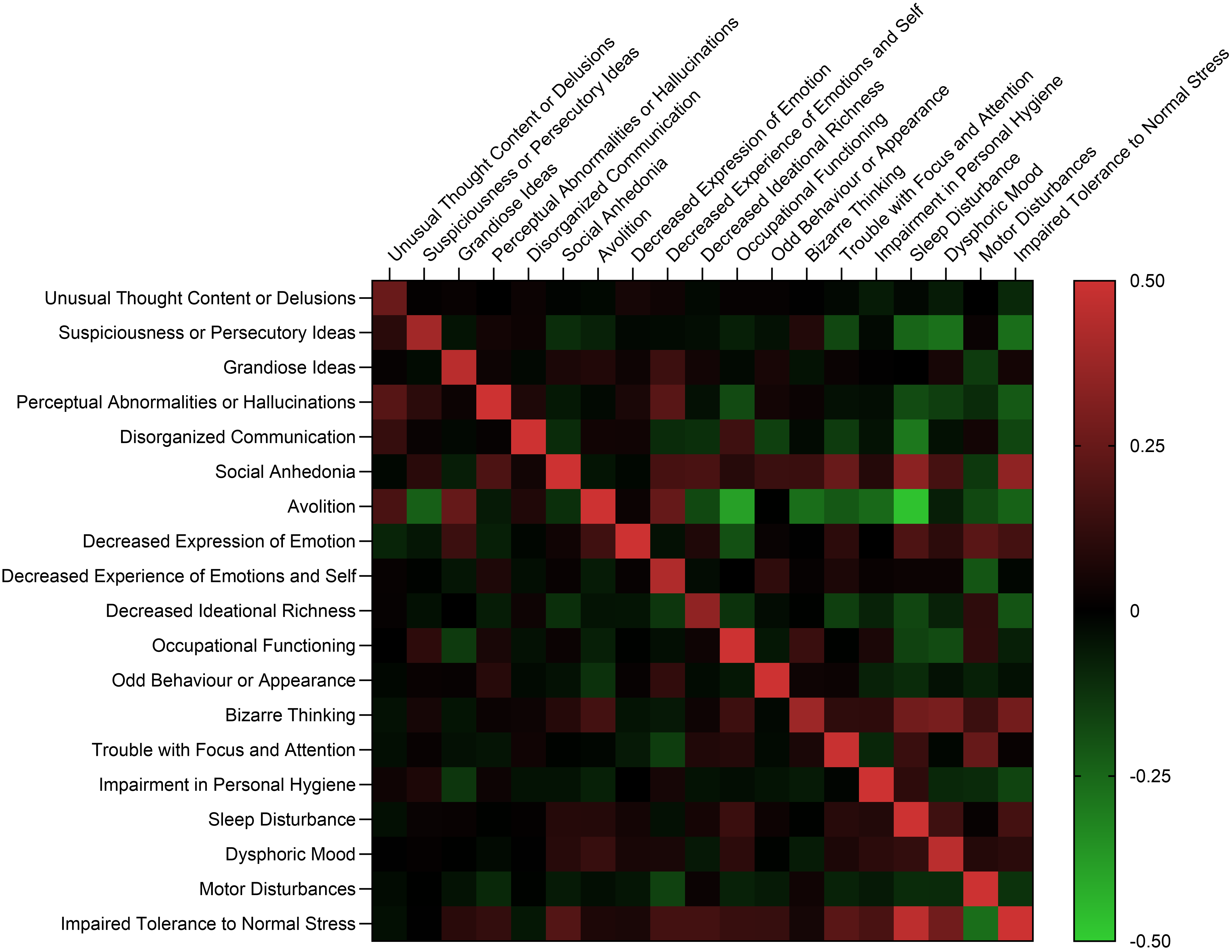}
    \caption{\textbf{Example heatmap from NAPLS-3.}
Standardized REFINE decoder coefficients for month 2 SOPS prediction. Rows denote stabilized baseline item coordinates and columns denote month 2 follow-up items. The strong diagonal structure reflects preservation of item-level correspondence, while off-diagonal entries highlight cross-symptom prognostic relationships.}
    \label{fig:NAPLS3_heatmap}
\end{figure}

For qualitative evaluation, a practicing psychiatrist reviewed the example coefficient heatmap in Figure~\ref{fig:NAPLS3_heatmap} for clinical plausibility and interpreted its dominant diagonal and off-diagonal patterns. The heatmap visualizes the REFINE linear decoder used to predict month 2 SOPS items from stabilized baseline SOPS item coordinates. Rows correspond to stabilized baseline items, columns correspond to month 2 follow-up items, and entries are standardized coefficients, so each value represents the expected change in the follow-up item, in standard deviation units, associated with a one-standard-deviation increase in the corresponding stabilized baseline item coordinate, conditional on the full stabilized baseline symptom profile. Consistent with REFINE's item-aligned construction, the largest positive weights tend to occur on the diagonal, indicating that REFINE preserves substantial item-level meaning while still allowing clinically informative cross-item associations.

Psychiatrically, the SOPS heatmap suggested several clinically meaningful profile contrasts in the psychosis prodrome. Baseline social anhedonia and impaired tolerance to normal stress predicted a broad array of later negative, disorganized, and general symptoms, suggesting that these baseline features may mark a broader non-positive-symptom vulnerability profile. This interpretation accords with prior work linking social anhedonia and impaired stress tolerance to broader clinical and functional burden in individuals at clinical high risk for psychosis \cite{Devylder13,Cressman15}. In contrast, avolition showed the strongest conditional contrast; although baseline avolition strongly predicted later avolition, it had negative coefficients for several later disorganized and general symptoms. This pattern suggests that REFINE separated a relatively specific amotivational presentation from a broader stress-sensitive or disorganized presentation, in line with prior work distinguishing volitional negative symptoms from other negative-symptom dimensions \cite{Azis19}. Together, these patterns illustrate how clinicians and researchers can read REFINE coefficients as a compact map of conditional symptom-profile trajectories, identifying baseline presentations that forecast symptom persistence, broad downstream vulnerability, or clinically meaningful contrast patterns.

\begin{figure}
    \centering
    \includegraphics[width=1\linewidth]{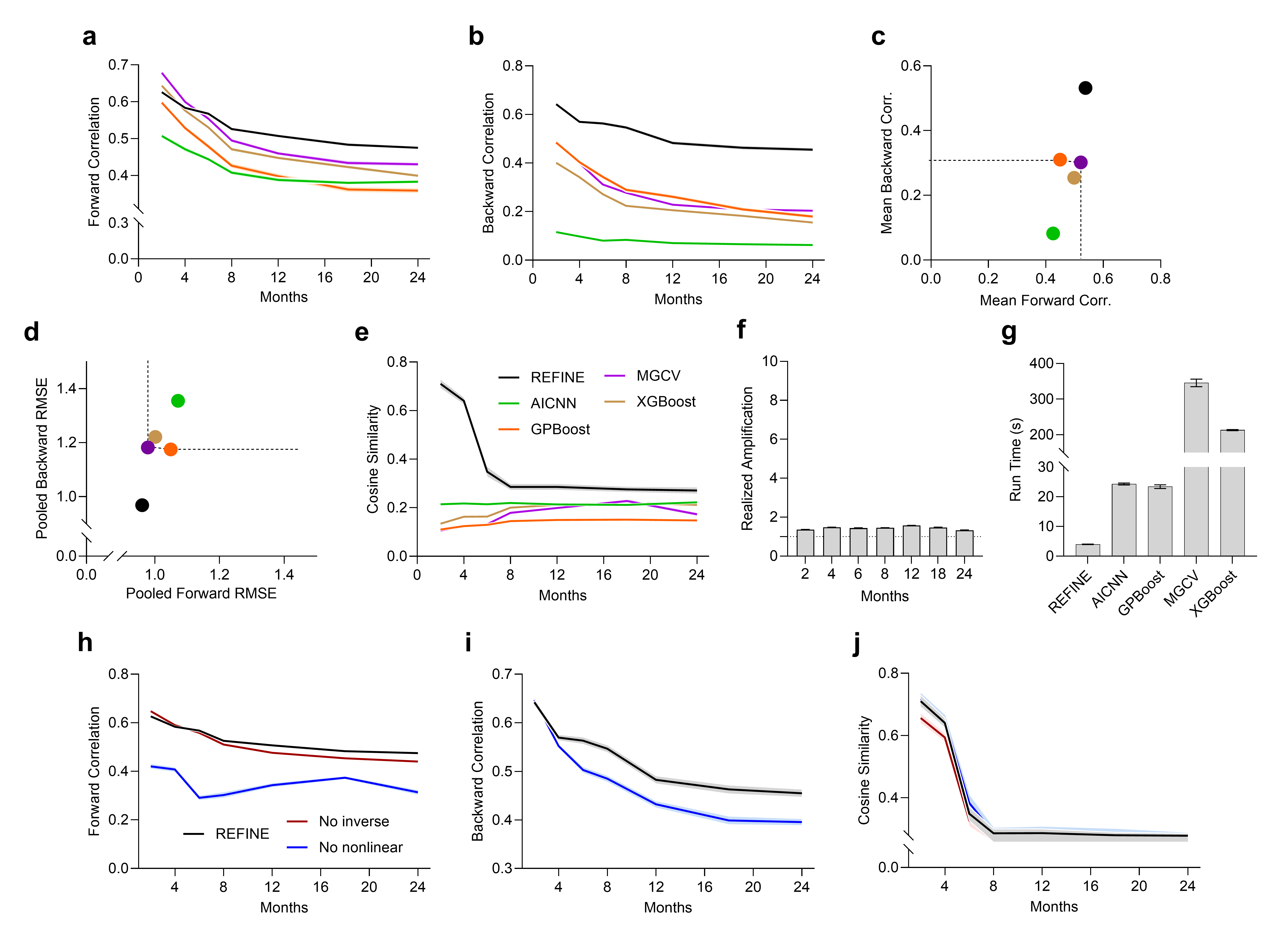}
\caption{\textbf{NAPLS-3 results across follow-up months.}
(a) Forward correlation between the predicted and observed SOPS items.
(b) Backward correlation, which quantifies how accurately follow-up items reconstruct baseline-derived representations.
(c) Pareto frontier summarizing the trade-off between forward and backward correlation.
(d) Pareto frontier with forward and backward pooled RMSE.
(e) Cosine similarity between each method’s contribution matrix and its diagonal component, indicating item alignment.
(f) Realized amplification induced by REFINE's inverse decoder.
(g) Runtime for each method.
(h--j) Ablation results for forward correlation, backward correlation, and cosine similarity, respectively. The black line overlaps the red line in (i), rendering the red line not visible. Panels (a--e) use a common color mapping for the primary methods, and panels
(h--j) use a separate common color mapping for the ablation variants. Error bands/bars represent 95\% confidence intervals, but some are barely visible due to the large number of bootstrap replications. }
    \label{fig:NAPLS3}
\end{figure}

Figure~\ref{fig:NAPLS3} summarizes the longitudinal prediction results. REFINE achieved the highest forward correlation at most follow-up time points (Figure~\ref{fig:NAPLS3}a) and the highest backward correlation at every time point (Figure~\ref{fig:NAPLS3}b), indicating that it preserves more of the shared longitudinal structure than the competing methods. Consistent with this pattern, REFINE occupied the most favorable region of the forward--backward correlation trade-off in the Pareto frontier plot (Figure~\ref{fig:NAPLS3}c). Similar results held with the forward and backward RMSE (Figure~\ref{fig:NAPLS3}d). Replacing the random forest preprocessor with bagged $k$-nearest neighbors or a bagged Nadaraya--Watson estimator did not materially change the results (Figure~\ref{fig:nonlin_swap}a).

REFINE also yielded the highest cosine similarity between each method’s baseline-to-follow-up contribution matrix and its diagonal component, particularly at month 2 (Figure~\ref{fig:NAPLS3}e). This result suggests that REFINE concentrates predictive contributions on item-aligned effects---that is, baseline item $j$ most strongly influences follow-up item $j$---rather than relying on diffuse cross-item mixing. As expected, cosine similarity decreased at longer prediction horizons, because the same SOPS items measured at more distant time points do not reflect exactly the same underlying symptom state as they do near baseline. We therefore conclude that, among the methods evaluated, REFINE most effectively preserves item-level meaning. 

Next, the inverse decoder produced only modest realized amplification, with values between 1 and 2 (Figure~\ref{fig:NAPLS3}f), suggesting that it did not substantially magnify variation in REFINE's learned stabilized representation or exhibit the extreme sensitivity expected from an unstable inversion. Finally, REFINE is computationally efficient: it is the fastest method in this comparison and completed consistently in under five seconds (Figure~\ref{fig:NAPLS3}g).

The ablation results further clarified which components drive performance. REFINE achieved the highest forward correlation across ablations, including the variant that replaces matrix inversion with a linear model trained on $\widehat{h}_t$ (Figure~\ref{fig:NAPLS3}h). REFINE also substantially outperforms OLS in backward correlation (Figure~\ref{fig:NAPLS3}i). For cosine similarity, REFINE performs consistently better when inverting $B_t$ than when training a linear model (Figure~\ref{fig:NAPLS3}j). Taken together, these results indicate that both the inversion step and the nonlinear component of REFINE are important for optimal performance.

\subsection{Treated Major Depression}

We next evaluated the algorithms on STAR*D, a large multistep clinical effectiveness trial in major depression that examined sequential treatment strategies (NDA ID: 2148) \citep{Rush04}. In Level 1 of STAR*D, treatment resistance was established by first administering citalopram to all participants. Participants who did not sufficiently improve in Level 1 could proceed to subsequent treatment steps, which included multiple treatment options and randomized comparisons in parts of the design. We therefore used Level 1 data to assess how well baseline features predict symptom trajectories during citalopram treatment. The dataset included 3,697 patients with complete baseline Quick Inventory of Depressive Symptomatology–Self-Report (QIDS-SR) data. Participants were followed for up to 16 weeks. Similar to the prior analysis, we used all 16 questionnaire items and all available time points, while conditioning on age and sex assigned at birth in $\bm{Z}$.

A practicing psychiatrist interpreted an example heatmap of standardized REFINE coefficients for predicting week 4 QIDS-SR items from baseline (Figure~\ref{fig:STARD_heatmap}). For this analysis, we collapsed the appetite-decrease and appetite-increase items into a single directional appetite-change score, and similarly collapsed the weight-decrease and weight-increase items into a single directional weight-change score, with decreases coded lower and increases coded higher. Psychiatrically, the heatmap suggested several clinically meaningful profile contrasts. Baseline concentration/decision-making had negative coefficients for later low energy/fatigue, psychomotor slowing, and loss of interest. One plausible interpretation is that REFINE separated a cognitive, ruminative, or decisional depressive presentation from a more anergic-neurovegetative presentation, rather than reducing concentration problems to a nonspecific marker of overall depressive severity. Prior work distinguishing cognitive from somatic depressive symptom dimensions supports this interpretation \cite{Duivis13}. Early-morning insomnia also contrasted negatively with middle insomnia. Because sleep-onset, sleep-maintenance, and early-morning-awakening insomnia represent separable sleep-symptom subtypes, this pattern suggests that REFINE captured clinically meaningful heterogeneity within sleep-related depressive symptoms \cite{Bjoroy20}. Finally, loss of interest did not simply track low energy/fatigue: the two symptoms showed separable coefficient patterns, aligning with prior evidence that anhedonia and fatigue overlap but remain partly distinguishable constructs \cite{Billones20}.

\begin{figure}
    \centering
    \includegraphics[width=0.8\linewidth]{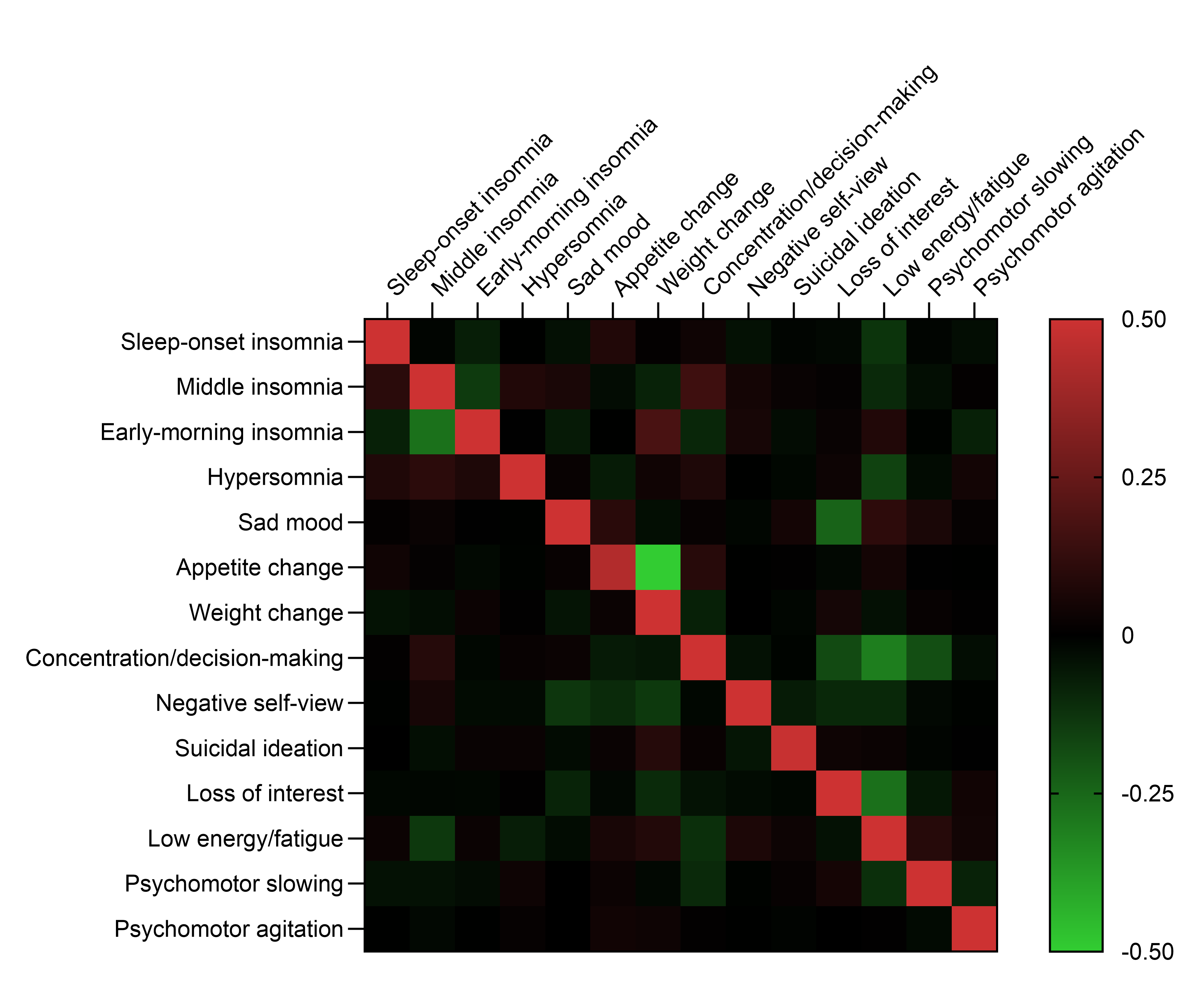}
    \caption{\textbf{Example heatmap from STAR*D.}
Standardized REFINE decoder coefficients for week 4 for QIDS-SR 16 prediction. We collapsed the appetite-decrease and appetite-increase items into a single directional appetite-change score, and similarly collapsed the weight-decrease and weight-increase items into a single directional weight-change score.}
    \label{fig:STARD_heatmap}
\end{figure}

\begin{figure}
    \centering
    \includegraphics[width=1\linewidth]{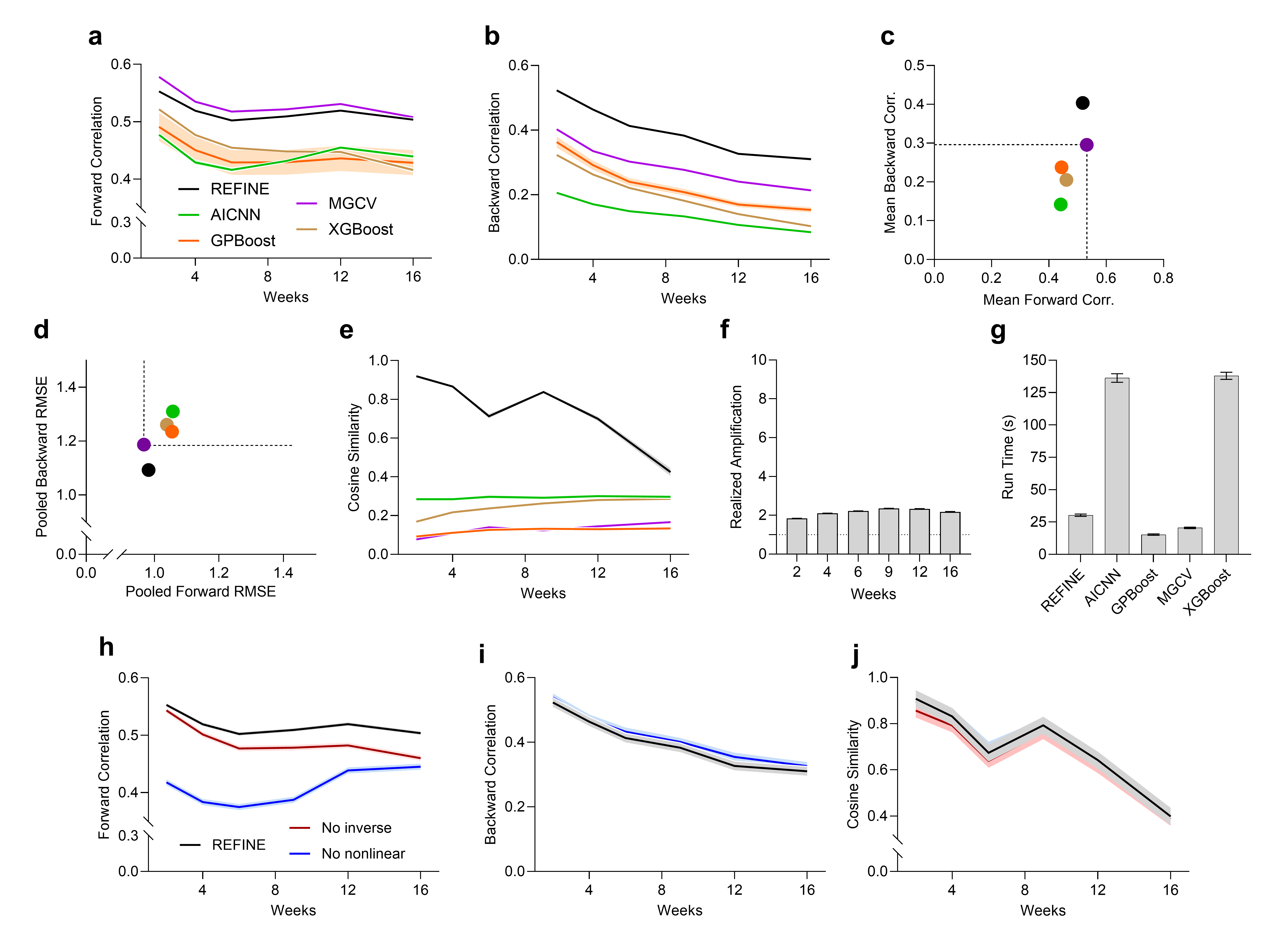}
    \caption{\textbf{STAR*D results across follow-up weeks.}
(a) Forward correlation,
(b) backward correlation,
(c) correlation Pareto plot,
(d) RMSE Pareto plot,
(e) cosine similarity,
(f) realized amplification,
(g) runtime per method,
(h) forward correlation of ablations,
(i) backward correlation of ablations,
(j) cosine similarity of ablations.}
    \label{fig:STARD}
\end{figure}

Figure \ref{fig:STARD} summarizes the quantitative results. REFINE outperformed all methods except MGCV in forward correlation (Figure \ref{fig:STARD}a). GPBoost exhibited high variability, reflecting both explicit convergence failures and substantial variability in fitted solutions even among runs that converged. In backward correlation, REFINE outperformed all comparator algorithms by a wide margin (Figure \ref{fig:STARD}b). The Pareto frontier also showed clear improvement for REFINE relative to competing methods (correlation in Figure \ref{fig:STARD}c, RMSE in Figure \ref{fig:STARD}d). Changing the nonlinear learner did not substantively alter the results (Figure~\ref{fig:nonlin_swap}b). In addition, REFINE achieved substantially higher cosine similarity than the other algorithms (exceeding 0.8), particularly at the earliest follow-up visits (weeks 2 and 4; Figure \ref{fig:STARD}e). These results indicate that REFINE again best preserved both longitudinal redundancy and item-level meaning. Moreover, the inverse decoder produced only modest realized amplification, with values near 2 across all follow-up time points (Figure~\ref{fig:STARD}f), while REFINE completed in approximately 30 seconds on average (Figure~\ref{fig:STARD}g).
Ablation results further identified the components necessary for REFINE’s performance (Figure~\ref{fig:STARD}h--j): both nonlinearity and inversion maximized forward correlation, and the inversion step also improved preservation of item-level meaning with the highest cosine similarity. We conclude that all of the results on STAR*D closely mimic the results seen in the NAPLS-3 dataset.

\subsection{Longitudinal Adolescent Health}

We finally evaluated the algorithms on an additional non-psychiatric dataset comprising 3,114 adolescents, focusing on the longitudinal prediction of anthropometric measures, blood pressure, and pubertal development (Appendix~\ref{supp:ABCD}). REFINE again delivered the best overall balance of predictive accuracy, preservation of item-level meaning, and runtime across comparator algorithms and ablated variants, indicating that its advantages extend beyond psychiatric questionnaires.

\section{Discussion}

In many psychiatric settings, total scores offer a useful summary but can obscure clinically meaningful heterogeneity in how individual symptoms change over time. To address this limitation, we developed REFINE, an algorithm for interpretable forecasting of the full symptom profile rather than a single aggregate endpoint. REFINE improves interpretability by unifying two components that are often treated separately in predictive modeling: preprocessing and model training. It places nonlinear flexibility in a learned preprocessing map that generates stabilized, item-aligned scores, and then applies an exactly linear decoder for prediction, yielding a globally interpretable prognostic model. This design mirrors a common scientific workflow---use flexible preprocessing to reduce measurement artifacts, then fit a simple model for transparent attribution. Formally, REFINE factorizes the conditional mean $\mathbb{E}(\bm X_t \mid \bm X_0, \bm Z)$ into a nonlinear preprocessor $h_t$ and a linear decoder $\beta_t$, where $h_t$ is learned from longitudinal redundancy rather than hand-crafted, instrument-specific rules.

Because psychiatric questionnaires lack an obvious, defensible preprocessing operator, we designed the preprocessor around two minimal criteria for clinically interpretable preprocessing: it must (i) exploit longitudinal redundancy by emphasizing temporally reproducible signal and (ii) preserve item-wise meaning by remaining aligned with the original questionnaire coordinates. Under these constraints, our theoretical results show that REFINE attains Bayes-optimal prediction and is uniquely identified within the class of preprocess--then--linearly--predict pipelines. We also showed that the end-to-end convergence rate is governed by the slower of the nonlinear estimator used for $h_t$ and the usual root-$n$ rate of the linear decoder. In our experiments, REFINE was the only method we evaluated that simultaneously matched or exceeded the predictive accuracy of other interpretable baselines while preserving item-level meaning. Taken together, the effort to formalize preprocessing for psychiatric questionnaires led us to an unexpected, generalizable design principle: learn task-aligned stabilization as an explicit preprocessing step, and reserve prediction for an exactly linear decoder to preserve transparent attribution.

We emphasize that REFINE offers a distinct route to interpretability compared to other approaches that append a linear explanation layer to an otherwise flexible predictor. Methods such as concept bottleneck \citep{Koh20} and varying-coefficient models \citep{Hastie93} typically construct intermediate representations or instance-specific effects that are not constrained to remain item-aligned, and therefore answer a different interpretive question than ours. By contrast, REFINE uses nonlinearity only to stabilize measurements in the original item space and then expresses the prognostic relationship as an exactly linear map. Relative to AICNN, this separation avoids an inherent accuracy--interpretability trade-off because the conditional mean admits an exact decomposition into a nonlinear preprocessor and a linear decoder. As a result, REFINE \emph{simultaneously} achieves Bayes-optimal predictive performance and linear interpretability.

REFINE also has limitations that motivate future work. First, our current implementation estimates $\beta_t$ by inverting $B_t$ separately at each time point, without borrowing strength across $t$; while this choice supports root-$n$ consistent estimation, it may increase variance when follow-up sample sizes are uneven or $B_t$ is ill-conditioned. Structured sharing across time (e.g., shrinkage toward a common decoder) could improve robustness while preserving exact linear decoding. Second, we instantiate $h_t$ with random forests, which are well suited to the small-to-moderate tabular regimes typical of questionnaire studies; larger datasets or richer covariates may benefit from alternative nonlinear learners, including deep models. Third, although we define $B_t$ via least-squares reconstruction, our characterization results allow broader linear operators; exploring regularized or domain-informed choices of $B_t$ may improve conditioning, sample efficiency, and robustness. Future work should therefore prioritize principled ways to borrow strength across time points, evaluate alternative nonlinear estimators of $h_t$ that preserve item alignment, and investigate regularized or domain-informed constructions of $B_t$ that improve conditioning and robustness.

In summary, REFINE offers a new, clinically compatible strategy for interpretable prediction: learn preprocessing from repeated-measurement redundancy, and ensure global interpretability by keeping the prognostic relationship exactly linear.

\section*{Statements and Declarations}

\subsection*{Competing Interests} The author declares no competing interests.

\section*{Author Contributions}
E.V.S. conceived the study, developed the methodology and software, performed all analyses, prepared the figures, interpreted the results, and wrote and revised the manuscript.

\subsection*{Funding} No external funding was received for this study.

\subsection*{Data Availability} NAPLS-3 and STAR*D data are available to qualified investigators through the NIMH Data Archive (NDA) under controlled access. ABCD Study data are available to qualified investigators through the NIH Brain Development Cohorts (NBDC) Data Hub under controlled access. Code implementing REFINE is publicly available at \url{https://github.com/ericstrobl/REFINE}.

\bibliography{biblio}

\appendix

\numberwithin{equation}{section}

\section{Appendix}
\subsection{Proofs} \label{sup:proofs}

\begin{reptheorem}{thm:preproc}
Assume $B_t$ and $\mathrm{Cov}(m_t(\bm X_0,\bm Z))$ are invertible, where
$m_t(\bm X_0,\bm Z):=\mathbb E(\bm X_t\mid \bm X_0,\bm Z)$.
Let $g:\mathbb R^{d+q}\to\mathbb R^{d}$ denote the preprocessor and let
$\beta\in\mathbb R^{d\times d}$. Then
$g(\bm X_0,\bm Z)\beta = m_t(\bm X_0,\bm Z)$ almost surely (a.s.)
if and only if $B_t\beta$ is invertible and
\[
g(\bm X_0,\bm Z) = m_t(\bm X_0,\bm Z)\,(B_t\beta)^{-1}B_t \quad\text{a.s.}
\]
\end{reptheorem}

\begin{proof}
Assume $g(\bm X_0,\bm Z)\beta=m_t(\bm X_0,\bm Z)$ almost surely. Since $\mathrm{Cov}(m_t(\bm X_0,\bm Z))$ is invertible, $m_t(\bm X_0,\bm Z)$ cannot lie in a proper subspace of $\mathbb R^d$ almost surely. Hence $\beta$ must be invertible, and therefore
\[
g(\bm X_0,\bm Z)=m_t(\bm X_0,\bm Z)\,\beta^{-1}\quad\text{a.s.}
\]
Because $B_t$ is invertible, $B_t\beta$ is invertible as well and
\[
m_t(\bm X_0,\bm Z)\,\beta^{-1}
=
m_t(\bm X_0,\bm Z)\,(B_t\beta)^{-1}B_t,
\]
so the stated representation follows. The converse is immediate by multiplication.
\end{proof}

\begin{repcorollary}{cor:unique}
Under the assumptions of Theorem~\ref{thm:preproc} and $B_t \beta = I_d$, Bayes recovery selects a unique pair $(g,\beta)$, namely $(h_t,B_t^{-1})$.
\end{repcorollary}

\begin{proof}
By Theorem~\ref{thm:preproc}, any Bayes-recovering pair satisfies
\[
g(\bm X_0,\bm Z)
=
m_t(\bm X_0,\bm Z)\,(B_t\beta)^{-1}B_t.
\]
With \(B_t\beta=I_d\), we have
\[
\beta = B_t^{-1},
\qquad
g(\bm X_0,\bm Z)=m_t(\bm X_0,\bm Z)B_t.
\]
We also have $m_t(\bm X_0,\bm Z)B_t
=
\mathbb E(\bm X_t\mid \bm X_0,\bm Z)B_t
=
\mathbb E(\bm X_tB_t\mid \bm X_0,\bm Z)
=
h_t(\bm X_0,\bm Z).$
So \((g,\beta)=(h_t,B_t^{-1})\) uniquely.
\end{proof}

\begin{replemma}{lem:invbeta}
Assume the following moment and regularity conditions for a fixed $t$: 
(i) $\mathbb{E}\|h_i\|_2^4<\infty$, $\mathbb{E}\|U_i\|_2^4<\infty$ and $\mathbb E(U_i\mid \bm X_{0,i},\bm Z_i)=0$; 
(ii) $\mathbb{E}(h_i^\top h_i)$ is positive definite; 
(iii) $a_{n,t} \in o_p(1)$; and
(iv) $\mathbb{E}\|\bm X_{0,i}\|_2^4<\infty$, $\mathbb{E}\|\bm X_{t,i}\|_2^4<\infty$, and $\mathbb{E}(\bm X_{t,i}^\top \bm X_{t,i})$ is positive definite.
Then:
\[
\|\widehat\beta_{t,\mathrm{OLS}}-\beta_t\|_F
\in
O_p\!\big(n_t^{-1/2} + a_{n,t}\big),
\qquad
\|\widehat\beta_{t,\mathrm{inv}}-\beta_t\|_F
\in
O_p(n_t^{-1/2}).
\]
\end{replemma}
\begin{proof}
For the first expression, we can write
\[
\widehat\beta_{t,\mathrm{OLS}}-\beta_t
=
\Big(\tfrac{1}{n_t}\widehat H_t^\top \widehat H_t\Big)^{-1}
\Big(\tfrac{1}{n_t}\widehat H_t^\top \bm X_t\Big)
-\beta_t.
\]
Inserting the identity
\[
\beta_t
=
\Big(\tfrac{1}{n_t}\widehat H_t^\top \widehat H_t\Big)^{-1}
\Big(\tfrac{1}{n_t}\widehat H_t^\top \widehat H_t\Big)\beta_t
\]
and subtracting gives
\[
\widehat\beta_{t,\mathrm{OLS}}-\beta_t
=
\Big(\tfrac{1}{n_t}\widehat H_t^\top \widehat H_t\Big)^{-1}
\Big(\tfrac{1}{n_t}\widehat H_t^\top(\bm X_t-\widehat H_t\beta_t)\Big).
\]
Define the residual
$U_{\widehat h,t}:=\bm X_t-\widehat H_t\beta_t$
so that
\[
\widehat\beta_{t,\mathrm{OLS}}-\beta_t
=
\Big(\tfrac{1}{n_t}\widehat H_t^\top \widehat H_t\Big)^{-1}
\Big(\tfrac{1}{n_t}\widehat H_t^\top U_{\widehat h,t}\Big).
\]
Write $\widehat H_t=H_t+R_t$ and substitute:
\[
U_{\widehat h,t}
=
\bm X_t - \widehat H_t \beta_t
=
\bm X_t - (H_t + R_t)\beta_t
=
(\bm X_t - H_t\beta_t) - R_t\beta_t.
\]
Let  
$U_t := \bm X_t - H_t\beta_t$ denote the $n_t\times d$ residual matrix, so
$U_{\widehat h,t} = U_t - R_t\beta_t.$ The $i$th row of $U_t$ equals
$U_i^\top=(\bm X_{t,i}-h_i\beta_t)^\top$. Expanding yields
\[
\tfrac{1}{n_t}\widehat H_t^\top U_{\widehat h,t}
=
\tfrac{1}{n_t}H_t^\top U_t
+\tfrac{1}{n_t}R_t^\top U_t
-\tfrac{1}{n_t}H_t^\top R_t\,\beta_t
-\tfrac{1}{n_t}R_t^\top R_t\,\beta_t.
\]
The first term is $O_p(n_t^{-1/2})$ by the central limit theorem since
$\mathbb E(U_i\mid \bm X_{0,i},\bm Z_i)=0$ and, under Assumption~(i),
$\mathbb E\|h_i\|_2^4<\infty$ and $\mathbb E\|U_i\|_2^4<\infty$.

For the remaining terms, by Cauchy--Schwarz,
\begin{align*}
\Big\|\tfrac{1}{n_t}R_t^\top U_t\Big\|_F
&\le
\|R_t\|_{2,n_t}\,\|U_t\|_{2,n_t},\\
\Big\|\tfrac{1}{n_t}H_t^\top R_t\Big\|_F
&\le
\|H_t\|_{2,n_t}\,\|R_t\|_{2,n_t},\\
\Big\|\tfrac{1}{n_t}R_t^\top R_t\Big\|_F
&\leq
\|R_t\|_{2,n_t}^2.
\end{align*}
Under the stated fourth-moment assumptions, $\|H_t\|_{2,n_t}^2
=
\frac{1}{n_t}\sum_{i=1}^{n_t}\|h_i\|_2^2 \in O_p(1)$ and $\|U_t\|_{2,n_t}^2
=
\frac{1}{n_t}\sum_{i=1}^{n_t}\|U_i\|_2^2\in O_p(1)$, while
$\|R_t\|_{2,n_t}=a_{n,t}$ by definition. Hence
\[
\Big\|\tfrac{1}{n_t}R_t^\top U_t\Big\|_F \in O_p(a_{n,t}),\qquad
\Big\|\tfrac{1}{n_t}H_t^\top R_t\Big\|_F \in O_p(a_{n,t}),\qquad
\Big\|\tfrac{1}{n_t}R_t^\top R_t\Big\|_F \in O_p(a_{n,t}^2).
\]
Moreover, since $\mathbb E(h_i^\top h_i)$ is positive definite by Assumption (ii) and $a_{n,t} \in o_p(1)$ by Assumption (iii), we have
$\big\|(\tfrac{1}{n_t}\widehat H_t^\top\widehat H_t)^{-1}\big\|_{\mathrm{op}}\in O_p(1)$.
Combining the above bounds gives
\[
\|\widehat\beta_{t,\mathrm{OLS}}-\beta_t\|_F
\in
O_p\!\big(n_t^{-1/2}+a_{n,t}+a_{n,t}^2\big)
\in
O_p\!\big(n_t^{-1/2}+a_{n,t}\big).
\]

For the second expression, $\widehat B_t$ is an ordinary least-squares estimator for the regression of $\bm X_0$ on $\bm X_t$, hence
$\|\widehat B_t-B_t\|_F\in O_p(n_t^{-1/2})$ under Assumption (iv). By the continuous mapping theorem $\widehat B_t^{-1}\to_p B_t^{-1}$ and hence
$\|\widehat B_t^{-1}\|_{\mathrm{op}} \in O_p(1)$.
Moreover,
\[
\widehat B_t^{-1}-B_t^{-1}
=
\widehat B_t^{-1}\bigl(B_t-\widehat B_t\bigr)B_t^{-1},
\]
so
\[
\|\widehat B_t^{-1}-B_t^{-1}\|_F
\le
\|\widehat B_t^{-1}\|_{\mathrm{op}}\,
\|\widehat B_t-B_t\|_F\,
\|B_t^{-1}\|_{\mathrm{op}}
=
O_p(n_t^{-1/2}).
\] 
Since $\widehat\beta_{t,\mathrm{inv}}=\widehat B_t^{-1}$ and $\beta_t=B_t^{-1}$, we conclude
$\|\widehat\beta_{t,\mathrm{inv}}-\beta_t\|_F=O_p(n_t^{-1/2})$
\end{proof}

\begin{reptheorem}{thm:convergence_rate}
Consider the same assumptions as Lemma~\ref{lem:invbeta} for a fixed $t$, and let
$\widehat\beta_t=\widehat\beta_{t,\mathrm{inv}}=\widehat B_t^{-1}$. Then
\[
\|\widehat H_t\,\widehat\beta_t - H_t\beta_t\|_{2,n_t}
\in
O_p\!\big(n_t^{-1/2}+a_{n,t}\big).
\]
\end{reptheorem}

\begin{proof}
Write $\widehat\beta_t = \beta_t + \Delta_\beta$ with $\Delta_\beta:=\widehat\beta_t-\beta_t$.
Using $\widehat H_t=H_t+R_t$,
\[
\widehat H_t\,\widehat\beta_t - H_t\beta_t
= R_t\beta_t + H_t\Delta_\beta + R_t\Delta_\beta.
\]
Applying the triangle inequality and $\|AB\|_{2,n_t}\le \|A\|_{2,n_t}\|B\|_{\mathrm{op}}$ gives
\[
\|\widehat H_t\,\widehat\beta_t - H_t\beta_t\|_{2,n_t}
\le
\|R_t\|_{2,n_t}\|\beta_t\|_{\mathrm{op}}
+\|H_t\|_{2,n_t}\|\Delta_\beta\|_{\mathrm{op}}
+\|R_t\|_{2,n_t}\|\Delta_\beta\|_{\mathrm{op}}.
\]
By definition $\|R_t\|_{2,n_t}=a_{n,t}$, and $\|H_t\|_{2,n_t}\in O_p(1)$ by Assumption (i).
Moreover, Lemma~\ref{lem:invbeta} yields $\|\Delta_\beta\|_{\mathrm{op}}\in O_p(n_t^{-1/2})$ for the inversion estimator.
Therefore the right-hand side is
$O_p\!\big(a_{n,t}+n_t^{-1/2}+a_{n,t}n_t^{-1/2}\big)=O_p\!\big(a_{n,t}+n_t^{-1/2}\big)$.
\end{proof}

\subsection{Time Complexity Analysis} \label{supp:time_complex}

Let \(T\) denote the number of follow-up times, let \(d\) be the number of questionnaire items, let \(q\) be the number of additional baseline covariates included in \(\bm Z\), and let \(p=d+q\) be the input dimension of the preprocessing map \(h_t\). For follow-up time \(t\), let \(n_t\) denote the number of subjects with both baseline and follow-up measurements. In the first stage, REFINE estimates \(B_t\) by (centered) least squares, which requires forming cross-products and then solving a \(d\times d\) linear system. Forming the cross-products costs \(O(n_t d^2)\), while solving for \(B_t\) costs \(O(d^3)\). The implementation then computes the inverse of \(B_t\) which costs at most \(O(d^3)\). Thus, the linear calibration step at each time point has complexity \(O(n_t d^2 + d^3)\).

After estimating \(B_t\), REFINE constructs the supervised proxy target \(\widebar{\bm X}_0^{(t)} = \bm X_t B_t\), which requires a matrix multiplication of cost \(O(n_t d^2)\). It then fits the nonlinear preprocessor \(h_t:(\bm X_0,\bm Z)\mapsto \widebar{\bm X}_0^{(t)}\) on \(n_t\) samples with input dimension \(p\) and output dimension \(d\). To keep the analysis independent of the specific nonlinear learner (e.g., random forests, boosting, or neural networks), let \(\mathcal{C}_{\mathrm{fit}}(n_t,p,d)\) denote the training time of the chosen multivariate nonlinear regression method. The total training complexity of REFINE is therefore
\[
\sum_{t=1}^{T} \Big(O(n_t d^2 + d^3) + \mathcal{C}_{\mathrm{fit}}(n_t,p,d)\Big),
\]
where the \(O(n_t d^2)\) term includes both least-squares cross-products and proxy construction. Let $n = \max_t n_t$, so that the above expression simplifies to
\[
O\!\left(T\big(n d^2 + d^3 + \mathcal{C}_{\mathrm{fit}}(n,p,d)\big)\right).
\]
We implement REFINE using random forest, where
\[
\mathcal{C}_{\mathrm{fit}}(n,p,d)
=
O\!\big(S\, \sqrt{p}\, d\, n\, \log n\big),
\]
and \(S\) is the number of trees. We conclude that our implementation of REFINE has time complexity
\[
O\!\left(T\big(n d^2 + d^3 + S\, \sqrt{p}\, d\, n\, \log n\big)\right).
\]

\subsection{Additional Results} \label{supp:ABCD}

We evaluated all algorithms on the Adolescent Brain Cognitive Development (ABCD) Study, using baseline physical health measures to predict longitudinal health \citep{Garavan18}. The target variables included anthropometric measures (height, weight, waist circumference), cardiovascular measures (systolic blood pressure, diastolic blood pressure, heart rate), and pubertal indicators (growth spurt, body hair, acne, blood DHEA, and blood testosterone). A total of 3,114 adolescents had complete baseline data and were reassessed every two years over a six-year period. The prediction task was to estimate future values of these measures from baseline observations.

We plot a heatmap of standardized REFINE coefficients for predicting year-two adolescent health measures from baseline measurements (Figure~\ref{fig:ABCD_heatmap}). Height and weight were by far the strongest predictors. Baseline height predicted later pubertal development measures and waist circumference, as expected given the close relationship between linear growth and pubertal maturation \cite{Tanner76}. In ABCD, pubertal stage variables were derived from the Pubertal Development Scale: stage 1 indicates prepubertal development, stage 2 indicates early puberty, and stage 3 indicates mid-puberty. As a sanity check, earlier pubertal-development stages also predicted later pubertal-development stages, reflecting the expected ordinal progression of pubertal maturation \cite{Herting21}. Interestingly, after accounting for height and waist circumference, baseline weight had negative coefficients for transitions into later pubertal-development stages. This pattern is plausible given prior evidence that adiposity has complex and sex-dependent associations with pubertal timing, including earlier pubertal development in girls but delayed or altered pubertal progression in boys \cite{Burt10}.

\begin{figure}
    \centering
    \includegraphics[width=0.75\linewidth]{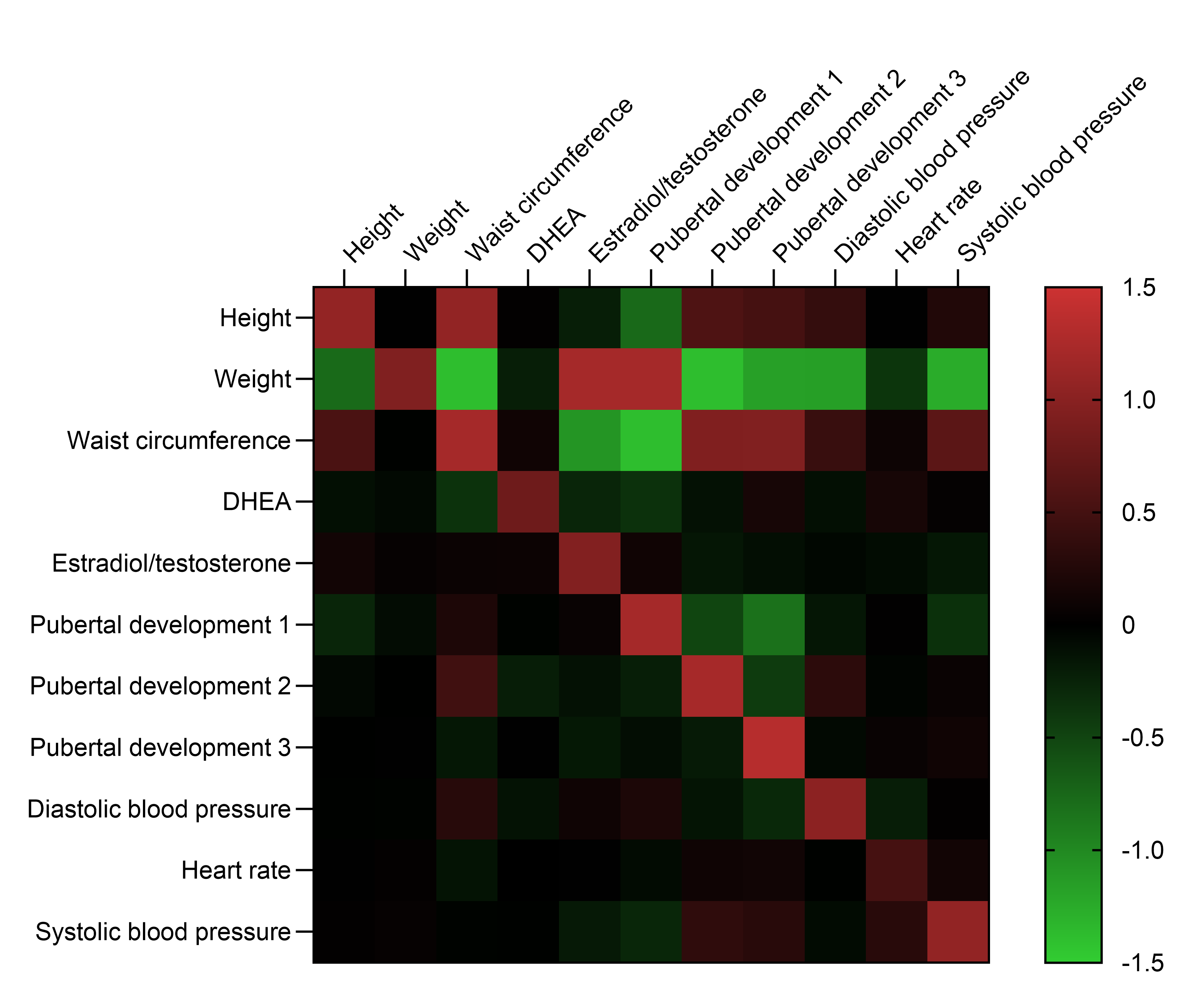}
    \caption{\textbf{Example heatmap from ABCD.}
Standardized REFINE decoder coefficients for predicting year-two adolescent health outcomes from baseline measures.}
    \label{fig:ABCD_heatmap}
\end{figure}

Figure~\ref{fig:ABCD} summarizes the results. REFINE achieved the second-highest forward correlation, narrowly trailing MGCV (Figure~\ref{fig:ABCD}a). REFINE substantially outperformed MGCV and all other methods in backward correlation (Figure~\ref{fig:ABCD}b), which moved REFINE onto the Pareto frontier (correlation in Figure~\ref{fig:ABCD}c, RMSE in Figure~\ref{fig:ABCD}d). Changing the nonlinear learner yielded similar results (Figure~\ref{fig:nonlin_swap}c). REFINE also achieved the highest cosine similarity at the first time point by a wide margin (Figure~\ref{fig:ABCD}e), indicating the strongest preservation of item-level meaning. In addition, the inverse decoder attained only modest realized amplification across all time steps (Figure~\ref{fig:ABCD}f), and REFINE had the lowest runtime among all methods (Figure~\ref{fig:ABCD}g). Taken together, these results indicate that REFINE provides the best overall trade-off among predictive accuracy, interpretability preservation, and computational efficiency.

\begin{figure}[!b]
\centering
\includegraphics[width=1\linewidth]{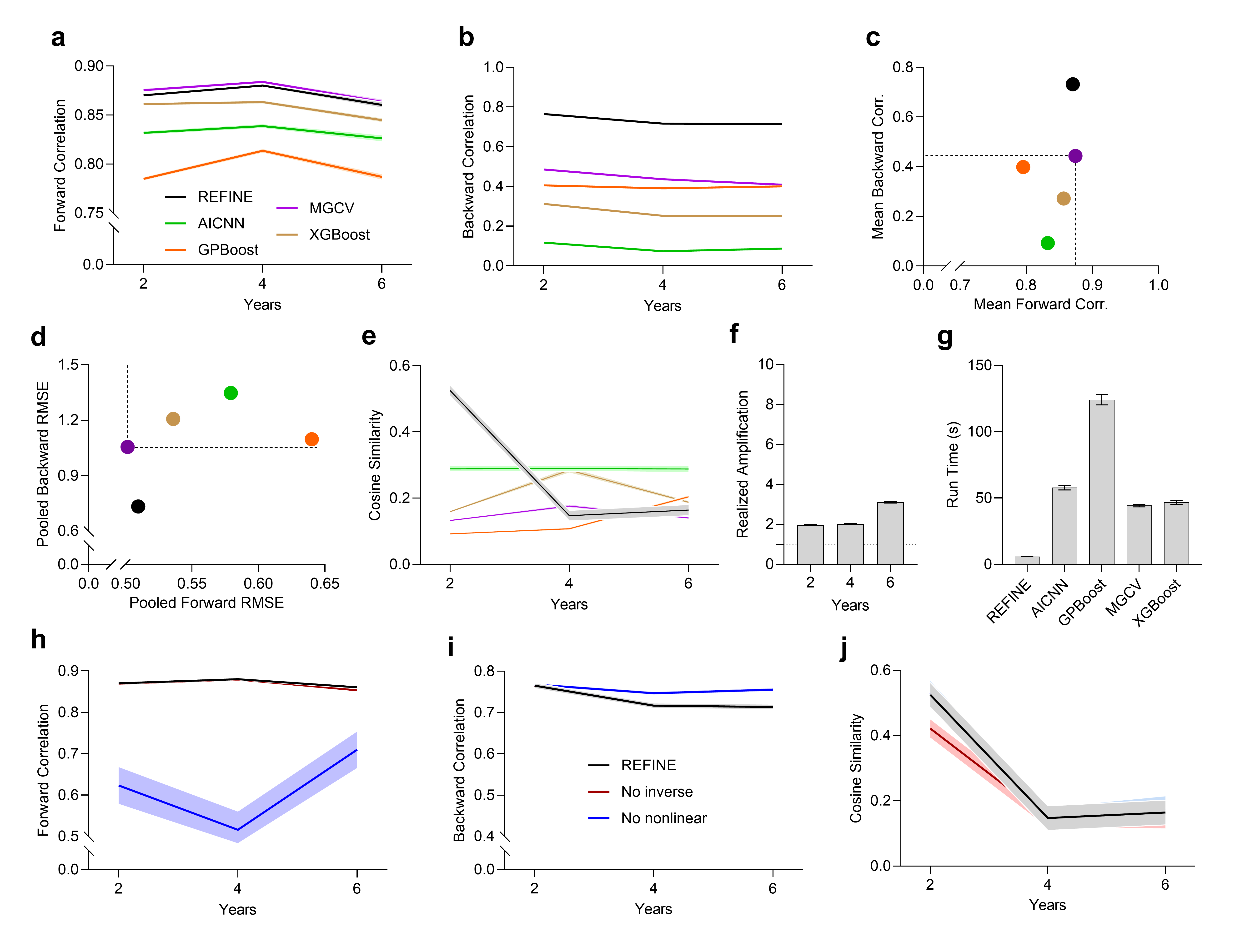}
\caption{\textbf{ABCD results across 6 years.} (a) Forward correlation,
(b) backward correlation,
(c) correlation Pareto plot,
(d) RMSE Pareto plot,
(e) cosine similarity,
(f) realized amplification,
(g) runtime per method,
(h) forward correlation of ablations,
(i) backward correlation of ablations,
(j) cosine similarity of ablations.}
\label{fig:ABCD}
\end{figure}

The ablation results showed a similar pattern. Removing the nonlinear component caused a substantial drop in forward correlation (Figure~\ref{fig:ABCD}h), highlighting the importance of nonlinear preprocessing for predictive performance. At the same time, all ablations maintained high backward correlation (Figure~\ref{fig:ABCD}i), indicating that nonlinear preprocessing does not compromise the ability to preserve item-level meaning. Finally, explicitly inverting the $B_t$ matrix consistently improved preservation of item-level meaning relative to fitting a final linear regression model on the nonlinear outputs (Figure~\ref{fig:ABCD}j). Overall, the ABCD results closely replicate the patterns observed in the NAPLS-3 and STAR*D datasets, supporting the conclusion that REFINE generalizes beyond psychiatric questionnaires to routine pediatric and public-health settings.

\begin{figure}[htbp]
    \centering
    \includegraphics[width=0.9\linewidth]{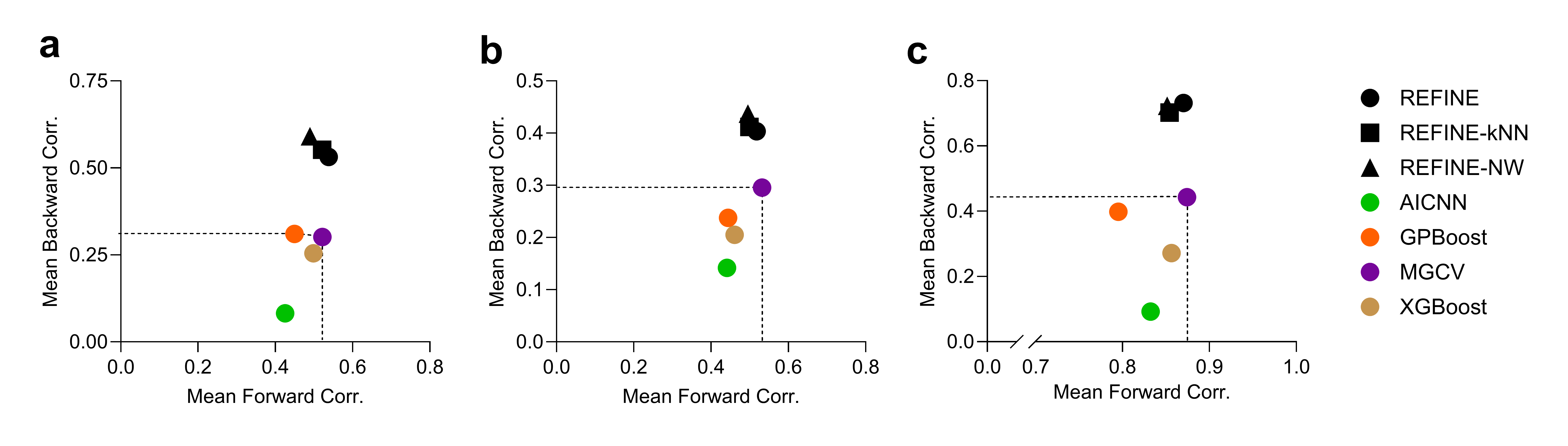}
    \caption{\textbf{Nonlinear learner sensitivity.} We replaced REFINE's default random forest preprocessor with two non-extrapolating nonlinear learners: bagged $k$-nearest neighbors and bagged Nadaraya--Watson estimators. Pareto frontier plots were similar across all three datasets: (a) NAPLS-3, (b) STAR*D, and (c) ABCD. This pattern suggests that REFINE's performance did not depend strongly on the specific nonlinear learner used for preprocessing.}
    \label{fig:nonlin_swap}
\end{figure}

\end{document}